\def\stdvers{0} \def\COLTvers{1}
\def\version{\stdvers}

\ifnum\version=\stdvers
\documentclass[11pt]{article}
\usepackage{amsthm}
\newcommand{\coltauthor}{}
\fi

\ifnum\version=\COLTvers
\documentclass[anon, 12pt]{colt2014}
\usepackage{times}
\fi

\usepackage{graphicx}
\usepackage{verbatim}
\usepackage{fullpage}
\usepackage{hyperref}
\usepackage{amssymb}
\usepackage{amsmath}
\usepackage{amsfonts}
\usepackage{algorithm}
\usepackage{ifthen}
\usepackage{dsfont}
\usepackage{algorithmic}
\usepackage{xfrac}
\usepackage{versions}

\ifnum\version=\stdvers
\includeversion{vstd}
\excludeversion{vCOLT}
\newtheorem{theorem}{Theorem}

\newtheorem{lemma}[theorem]{Lemma}
\newtheorem{definition}{Definition}

\newenvironment{innerproof}[1][Proof]
{\begin{proof}[#1]}
{\end{proof}}

\fi

\ifnum\version=\COLTvers
\includeversion{vCOLT}
\excludeversion{vstd}

\newenvironment{innerproof}[1][Proof]
{\par\noindent{\bfseries\upshape #1}}{\hfill$\blacktriangle$\\[2mm]}

\fi

\newtheorem*{claim*}{Claim}

\def\reals{\mathbb{R}}
\def\R{\reals}
\def\N{\mathbb{N}}

\def\eps{\epsilon}

\def\poly{\mathrm{poly}}

\def\Z{\mathbb{Z}}

\def\cum{\mathrm{cum}}
\def\Var{\mathrm{Var}}

\def\Poisson{\mathsf{Poisson}}
\def\Pois{\Poisson}
\def\Bin{\mathrm{Bin}}
\def\Multinom{\mathrm{Multinom}}
\def\Normal{\mathcal{N}}
\def\Norm{\Normal}

\def \deq{\sim}
\def\m{\mathrm{m}}
\def\ds{\rule{0pt}{1.5ex}}

\newcommand{\abs}[1]{\left|#1\right|}
\newcommand{\norm}[1]{\left\|#1\right\|}

\newcommand{\prob}[1]{{\mathsf{Pr}}\left(#1\right)}
\newcommand{\cumtns}[2]{\ensuremath{\kappa_{#1}^{#2}}}

\newcommand{\E}[1]{\mathbb{E}\left(#1\right)}

\newcommand{\spn}[1]{\mathrm{span}\left(#1\right)}

\newcommand{\diag}[1]{\mathrm{diag}\left( #1 \right)}

\renewcommand{\vec}[1]{\mathrm{vec}\left( #1 \right)}

\renewcommand{\dim}{n}
\newcommand{\means}{m}

\renewcommand{\d}{\operatorname{d}}

\newcommand{\Vr}[1]{\mathrm{Var}\left(#1\right)}
\newcommand{\Conjug}[1]{}
\newcommand{\suchthat}{\;\ifnum\currentgrouptype=16 \middle\fi|\;}
\newcommand{\angles}[1]{\left\langle #1 \right\rangle}
\newcommand{\dist}{\mathrm{dist}}

\newcommand{\email}[1]{\href{mailto:#1}{\texttt{#1}}}
\DeclareRobustCommand{\stirling}{\genfrac\{\}{0pt}{}}
\newcommand{\norms}[1]{{\lVert#1\rVert}^2}

\def\final{1}  

\ifnum\final=0  
\newcommand{\lnote}[1]{[{\small Luis: \textbf{#1}}]\marginpar{*}}
\newcommand{\nnote}[1]{[{\small Navin: \textbf{#1}}]\marginpar{*}}
\newcommand{\vnote}[1]{[{\small Jimmy: \textbf{#1}}]\marginpar{*}}
\newcommand{\anonnote}[1]{[{\small anon: \textbf{#1}}]\marginpar{*}}
\newcommand{\sidecomment}[1]{\marginpar{\tiny #1}}
\newcommand{\details}[1]{[[#1]]}
\else 
\newcommand{\lnote}[1]{}
\newcommand{\nnote}[1]{}
\newcommand{\vnote}[1]{}
\newcommand{\anonnote}[1]{}
\newcommand{\sidecomment}[1]{}
\newcommand{\details}[1]{}
\fi  

\allowdisplaybreaks

\title{The More, the Merrier:\\ the Blessing of Dimensionality for Learning  Large Gaussian Mixtures }
\date{}

\ifnum\version=\stdvers
\author{Joseph Anderson \\
Ohio State University \\
\email{andejose@cse.ohio-state.edu}
\and
Mikhail Belkin \\
Ohio State University \\
\email{mbelkin@cse.ohio-state.edu}
\and
Navin Goyal\\
Microsoft Research India \\
\email{navingo@microsoft.com}
\and
Luis Rademacher\\
Ohio State University\\
\email{lrademac@cse.ohio-state.edu}
\and
James Voss \\
Ohio State University\\
\email{vossj@cse.ohio-state.edu}
}
\fi

\ifnum\version=\COLTvers
\coltauthor{
  \Name{Joseph Anderson} \Email{andejose@cse.ohio-state.edu} \\
  \addr Ohio State University
      \AND
  \Name{Mikhail Belkin}   \Email{mbelkin@cse.ohio-state.edu} \\
  \addr Ohio State University \\
  Computer Science and Engineering, \\
  2015 Neil Avenue, Dreese Labs 597. \\
  Columbus, OH 43210
  \addr Ohio State University
      \AND
  \Name{Navin Goyal}     \Email{navingo@microsoft.com} \\
  \addr Microsoft Research India
      \AND
  \Name{Luis Rademacher}   \Email{lrademac@cse.ohio-state.edu} \\
  \addr Ohio State University \\
  Computer Science and Engineering, \\
  2015 Neil Avenue, Dreese Labs 495. \\
  Columbus, OH 43210
      \AND
  \Name{James Voss}  \Email{vossj@cse.ohio-state.edu} \\
  \addr Ohio State University \\
  Computer Science and Engineering, \\
  2015 Neil Avenue, Dreese Labs 586. \\
  Columbus, OH 43210
}

\fi

\begin{document}

\maketitle

\begin{abstract}
In this paper we show that very large mixtures of Gaussians are efficiently learnable in high dimension.
More precisely, we prove that a mixture with known identical covariance matrices whose number of components is a polynomial of any fixed degree in the dimension $n$ is polynomially learnable as long as a certain non-degeneracy condition on the means is satisfied. It turns out that this condition is generic in the sense of smoothed complexity, as soon as the dimensionality of the space is high enough. 
Moreover, we prove that no such condition can possibly exist in low dimension and the problem of learning the parameters is generically hard.  In contrast, much of the existing work on Gaussian Mixtures relies on low-dimensional projections and thus hits an artificial barrier. 

Our main result on mixture recovery relies on a new ``Poissonization"-based technique, which transforms a mixture of Gaussians to a linear map of a product distribution. The problem of learning this map can be efficiently solved using some recent results on tensor decompositions and Independent Component Analysis (ICA), thus giving an  algorithm for recovering the mixture. 
In addition, we combine our low-dimensional hardness results for Gaussian mixtures with  Poissonization to show how to embed difficult instances of low-dimensional Gaussian mixtures into the ICA setting, thus establishing exponential information-theoretic lower bounds for underdetermined ICA in low dimension. To the best of our knowledge, this is the first such result  in the literature.


In addition to contributing to  the problem of Gaussian mixture learning, we
 believe that this work is among the first steps toward better understanding the rare phenomenon of the ``blessing of dimensionality" in the computational aspects of statistical inference.
\processifversion{vCOLT}{\footnote{We would like to submit this paper for consideration for the Best Student Paper award.}}
\end{abstract}
\begin{vCOLT}
\begin{keywords}
  Gaussian mixture models, tensor methods, blessing of dimensionality, smoothed analysis, Independent Component Analysis
\end{keywords}
\end{vCOLT}


\section{Introduction}

The question of  recovering a probability distribution from a finite set of samples is one of the most fundamental questions of statistical inference. While classically such problems have  been considered in 
low dimension, more recently inference in high dimension has drawn significant attention in statistics and computer science literature. 

In particular, an active line of investigation in theoretical computer science has dealt with the question of learning a Gaussian Mixture Model in high dimension. This line of work was started in~\cite{dasgupta99} where  the first algorithm to recover parameters using a number of samples polynomial in the dimension was presented. The method relied on random projections to a low dimensional space and required certain separation conditions for the means of the Gaussians.  Significant work was done in order to weaken the separation conditions and to generalize the  result (see e.g.,~\cite{dasgupta00,arora01,vempala02, achlioptas05,feldman06}). 
Much of this work has polynomial sample and time complexity but requires strong separation conditions on the Gaussian components. A completion of the attempts to weaken the separation conditions was achieved 
in~\cite{BelkinFOCS10} and~\cite{moitra10}, where it was shown that arbitrarily small separation was sufficient for learning a general mixture with a fixed number of components in polynomial time. Moreover, a one-dimensional example given in~\cite{moitra10} showed  that  an exponential dependence on the number of components was unavoidable unless strong separation requirements were imposed.  Thus the question of polynomial learnability appeared to be settled.  It is worth noting that  while quite different in many aspects, all of these  papers used a general scheme similar to that in the original work~\cite{dasgupta00} by reducing  high-dimensional inference to a small number of low-dimensional problems through appropriate projections.

However,  a surprising result was recently proved in~\cite{HsuK13}. The authors showed that a 
mixture of $d$ Gaussians
 in dimension $d$ 
could be learned using a polynomial number of samples, assuming a non-degeneracy condition 
on the configuration of the means.  The result in~\cite{HsuK13} is inherently high-dimensional as that condition is 
never satisfied when  the means belong to  a lower-dimensional space. 
Thus the problem of learning a mixture gets progressively  computationally easier as the dimension increases, a ``blessing of dimensionality!" It is important to note that this was quite different from much of the previous work, which had primarily used projections to lower-dimension spaces. 

Still, there remained a large gap between the worst case impossibility of efficiently learning more than a fixed number of Gaussians  in low dimension and the situation when the number of components is equal to the dimension. Moreover, it was not completely clear whether the underlying problem was  genuinely easier in high dimension or our algorithms in low dimension were suboptimal. The one-dimensional example  in~\cite{moitra10} cannot answer this question as it is a specific worst-case scenario, which can be potentially ruled out by some genericity condition. 

In our paper we take a step to eliminate this gap by showing that even very large mixtures of Gaussians can be polynomially learned. More precisely, we show that a mixture of $m$ Gaussians with equal known covariance can be polynomially learned as long as $m$ is bounded from above by a polynomial of the dimension $n$ and a  certain more complex non-degeneracy condition for the means is satisfied. We show that if $n$ is high enough, these non-degeneracy conditions are  generic in the smoothed complexity sense. Thus for any fixed $d$, $O(n^d)$ generic Gaussians can be polynomially learned in dimension $n$.

Further, we prove that no such condition can exist in low dimension.  A  measure of non-degeneracy must be monotone in the sense that adding Gaussian components must make the condition number worse. 
However, we show that for $k^2$ points uniformly sampled from $[0,1]$ there are (with high probability) two mixtures of    unit Gaussians with means on non-intersecting subsets of these points, whose $L^1$ distance is $O^*(e^{-{k}})$ and which are thus not polynomially identifiable. More generally, in dimension $n$ the distance becomes  $O^*(e^{-\sqrt[n]{k}})$.  That is, the conditioning improves as the dimension increases, which is consistent with our algorithmic results.


To summarize, our contributions are as follows:
\begin{enumerate}
\item We show that for any $q$, a mixture of $n^q$ Gaussians in dimension $n$ can be learned in time and  number of samples polynomial in $n$ and a certain ``condition number" $\sigma$. We show that if the dimension is sufficiently high, this results in an algorithm polynomial from the smoothed analysis point of view (Theorem \ref{thm:correctness}). To do that we provide smoothed analysis of the condition number using certain results from~\cite{RudelsonVershynin} and anti-concentration inequalities.  
    The main technical ingredient of the algorithm  is a new ``Poissonization" technique to reduce Gaussian mixture estimation to a problem of recovering a linear map of a product distribution known as underdetermined Independent Component Analysis (ICA). 
We combine this with the recent work on efficient algorithms for underdetermined ICA from~\cite{GVX} to obtain the necessary bounds.

\item We show that in low dimension polynomial identifiability fails in a certain generic sense (see Theorem~\ref{thm:low-dim-identifiability}).  Thus the efficiency of our main algorithm is truly a consequence of the "blessing of dimensionality" and no comparable algorithm exists in low dimension. The analysis is based on  results from approximation theory and Reproducing Kernel Hilbert Spaces.

Moreover, we combine the approximation theory results with the Poissonization-based  technique to show how to embed difficult instances of low-dimensional Gaussian mixtures into the ICA setting, thus establishing exponential information-theoretic lower bounds for underdetermined Independent Component Analysis in low dimension. To the best of our knowledge, this is the first such result  in the literature.

\end{enumerate}

We discuss our main contributions more formally now.
The notion of Khatri--Rao power $A^{\odot d}$ of a matrix $A$ is defined in Section~\ref{sec:Prelims}.
\begin{theorem}[Learning a GMM with Known Identical Covariance] \label{thm:correctness}
Suppose $\means \geq \dim$ and
let $\epsilon, \delta > 0$.
Let $w_1 \Norm(\mu_1, \Sigma) + \ldots + w_m \Norm(\mu_m, \Sigma)$ be an $\dim$-dimensional GMM, i.e.
$\mu_i \in \R^n$, $w_i > 0$, and $\Sigma \in \R^{n \times n}$. 
Let $B$ be the $\dim \times \means$ matrix whose $i^{th}$ column is $\mu_i / \norm{\mu_i}$.
If there exists $d \in 2\N$ so that $\sigma_m\left(B^{\odot d/2} \right) > 0$, then Algorithm \ref{alg:reduction} recovers each $\mu_i$ to within $\epsilon$ accuracy with probability $1 - \delta$.
Its sample and time complexity are at most
\begin{align*}
\poly\left(m^{d^2},
      \sigma^{d^2}, u^{d^2}, w^{d^2},
      d^{d^2}, r^{d^2}, 1/\epsilon, 1/\delta, 1/b, \log^{d^2}\bigl({1}/{(b \eps \delta)}\bigr) \right)
\end{align*}
where 
$w \geq \max_{i}(w_i)/\min_{i}(w_i)$,
$u \geq \max_i \norm{\mu_i}$,
$r \geq \big(\max_i\norm{\mu_i} + 1)/(\min_i\norm{\mu_i})\big)$,
$0 < b \leq \sigma_m(B^{\odot d/2})$ are bounds provided to the algorithm, and
$\sigma = \sqrt{\lambda_{\max}(\Sigma)}$.
\end{theorem}


Given that the means have been estimated, the weights can be recovered using the tensor structure of higher order cumulants (see Section~\ref{sec:Prelims} for the definition of cumulants). 
This is shown in Appendix \ref{sec:WeightRecovery}.

We show that $\sigma_{\min}(A^{\odot d})$ is large in the smoothed analysis sense, namely, if we 
start with a base matrix $A$ and perturb each entry randomly to get $A'$, then
$\sigma_{\min}(A^{\odot d})$ is likely to be large. More precisely,

\begin{theorem}\label{thm:smoothed-sigma-min-intro}
For $n > 1$, let $M \in \R^{n \times \binom{n}{2}}$ be an arbitrary matrix. Let 
$N \in \R^{n \times \binom{n}{2}}$ 
be a randomly sampled matrix with each entry iid from $\Normal(0, \sigma^2)$, for 
$\sigma > 0$. Then, for some absolute constant $C$,
\begin{align*}
\prob{\sigma_{\min}((M+N)^{\odot 2}) \leq {\sigma^2}/{n^7}} \leq {2C}/{n}.
\end{align*}
\end{theorem}

We point out the simultaneous and independent work of \cite{DBLP:journals/corr/BhaskaraCMV13}, where the authors prove learnability results related to our Theorems~\ref{thm:correctness} and \ref{thm:smoothed-sigma-min-intro}. We now provide a comparison. The results in 
\cite{DBLP:journals/corr/BhaskaraCMV13}, which are based on tensor decompositions,  are stronger in that they can learn mixtures of axis-aligned Gaussians (with
non-identical covariance matrices) without requiring to know the covariance matrices in advance. Their results hold
under a smoothed analysis setting similar to ours. To learn a mixture of roughly $n^{\ell/2}$ Gaussians up to an 
accuracy of $\epsilon$ their algorithm has running time and sample complexity 
$\mathrm{poly}_\ell(n, 1/\epsilon, 1/\rho)$ and succeeds with probability at least $1- \exp(-Cn^{1/3^\ell})$, 
where the means are perturbed by adding an $n$-dimensional Gaussian from $\Normal(0, I_n \rho^2/n )$. 
On the one hand, the success probability of their algorithm is much better (as a function of $n$, exponentially close to $1$ as opposed to polynomially close to $1$, as in our result). 
On the other hand, this comes at a high price in terms of the running time and sample complexity: The polynomial $\mathrm{poly}_\ell(n, 1/\epsilon, 1/\rho)$ above has degree
exponential in $\ell$, unlike the degree of our bound which is polynomial in $\ell$.
Thus, in this respect, the two results can be regarded as incomparable points on an error vs running time 
(and sample complexity) trade-off curve. 
Our result is based on a reduction from learning GMMs to ICA  which could be of  independent interest given that both 
problems are extensively studied in somewhat disjoint communities.  The technique  of Poissonization is, to the best of our knowledge, new in the GMM setting.
Moreover, our analysis can be used in the reverse
direction to obtain hardness results for ICA.

Finally, in Section~\ref{identifiability_low_dimension} we show that in low dimension the situation is very different from the high-dimensional generic efficiency given by Theorems \ref{thm:correctness} and \ref{thm:smoothed-sigma-min-intro}: The problem is generically hard. More precisely, we show:
\begin{theorem}\label{thm:low-dim-identifiability}
Let $X$ be a set of $k^2$ points uniformly sampled from $[0,1]^n$. Then with high probability there exist two mixtures with equal number of unit Gaussians $p$, $q$ centered on disjoint subsets of $X$, such that, for some $C>0$,
\[
\|p-q\|_{L^1(\R^n)} < e^{-C \left({k}/{\log k}\right)^{{1}/{n}} }.
\]
\end{theorem}
Combining the above lower bound with our reduction provides a similar lower bound for ICA; 
see a discussion on the connection with ICA below.
Our lower bound gives an information-theoretic barrier. This is in contrast to conjectured computational 
barriers that arise in related settings based on the noisy parity problem (see \cite{HsuK13} for pointers).
The only previous information-theoretic lower bound for learning GMMs we are aware of is due to \cite{moitra10}
and holds for two specially designed one-dimensional mixtures.

\paragraph{Connection with ICA.}
A key observation of \cite{HsuK13} is that methods based on the higher order statistics used in Independent Component Analysis (ICA) can be adapted to the setting of learning a Gaussian Mixture Model.  In ICA, samples are of the form $X = \sum_{i=1}^m A_i S_i$ where the latent random variables $S_i$ are independent, and the column vectors $A_i$ give the directions in which each signal $S_i$ acts.  The goal is to recover the vectors $A_i$ up to inherent ambiguities. The ICA problem is typically posed when $m$ is at most the dimensionality of the observed space (the ``fully determined" setting), as recovery of the directions $A_i$ then allows one to demix the latent signals.  The case where the number of latent source signals exceeds the dimensionality of the observed signal $X$ is the \emph{underdetermined ICA} setting.%
\footnote{See \cite[Chapter 9]{ComonJutten} for a recent account of algorithms for underdetermined ICA.}
Two well-known algorithms for underdetermined ICA are given in \cite{FOOBI} and \cite{BIOME}. 
Finally, \cite{GVX} provides an algorithm with rigorous polynomial time and sampling bounds \vnote{same exponential caveats as our proposed technique in this paper} for underdetermined ICA in high dimension in the presence of Gaussian noise.

Nevertheless, our analysis of the mixture models can be embedded in ICA  to show exponential information-theoretic hardness of performing ICA in low-dimension, and thus establishing the blessing of dimensionality for ICA as well. 
\begin{theorem}\label{thm:low-dim-identifiability-ica}
Let $X$ be a set of $k^2$ random $n$-dimensional unit vectors.
Then with high probability, there exist two disjoint subsets of $X$, such that when these two sets form the columns of matrices $A$ and $B$ respectively, there exist noisy ICA models $AS+\eta$ and $BS' + \eta'$ which are exponentially close as distributions in $L^1$ distance and satisfying:  (1) The coordinate random variables of $S$ and $S'$ are scaled Poisson random variables.  For at least one coordinate random variable, $S_i = \alpha X$, where $X \sim \Poisson(\lambda)$ is such that $\alpha$ and $\lambda$ are polynomially bounded away from 0.  (2) The Gaussian noises $\eta$ and $\eta'$ have polynomially bounded directional covariances.
\end{theorem}
We sketch the proof of Theorem~\ref{thm:low-dim-identifiability-ica} in Appendix~\ref{sec:sketch-pf-thm-ICA-bound}.

\paragraph{Discussion.} Most problems become harder in high dimension, often exponentially harder, a behavior known as ``the curse of dimensionality."  Showing that a complex problem does not become exponentially harder often constitutes major progress in its understanding. In this work we demonstrate a reversal of this curse, showing that  the lower dimensional instances are exponentially harder than those in high dimension. This seems to be a rare situation in statistical inference and computation. In particular, while high-dimensional concentration of mass can sometimes be a blessing of dimensionality, in our case the generic computational efficiency of our problem comes from anti-concentration.

We hope that this work will enable better understanding of this unusual phenomenon and its applicability to a wider class of computational and statistical problems.

%


\section{Preliminaries}
\label{sec:Prelims}
The singular values of a 
matrix $A \in \R^{m \times n}$ will be ordered in the decreasing order: $\sigma_1 \geq
\sigma_2 \geq \dotsb \geq \sigma_{\min(m, n)}$. By $\sigma_{\min}(A)$ we mean $\sigma_{\min(m,n)}$. 

For a real-valued random variable $X$, the \emph{cumulants} of
$X$ are polynomials in the moments of $X$. For $j \geq 1$, the $j$th
cumulant is denoted $\cum_j(X)$. Denoting $\m_j := \E{X^j}$, we have, for
example: $\cum_1(X) = \m_1, \cum_2(X) = \m_2 - \m_1^2, \cum_3(X) =
\m_3 - 3 \m_2\m_1 + 2\m_1^3$. In general, cumulants can be defined 
as certain coefficients of a power series expansion of the logarithm of the characteristic function of $X$:
$
\log(\mathbb{E}_X(e^{itX})) = \sum_{j=1}^{\infty}\cum_j(X) \frac{(it)^j}{j!} 
$. 
The first two cumulants are the same as the expectation and the
variance, resp. Cumulants have the property that for two independent
random variables $X, Y$ we have $\cum_j(X+Y) = \cum_j(X) + \cum_j(Y)$ (assuming
that the first $j$ moments exist for both $X$ and $Y$).  Cumulants are degree-$j$ homogeneous, i.e.~if $\alpha \in \R$ and $X$ is a random variable, then
$\cum_j(\alpha X) = \alpha^j\cum_j(X)$.  The first two
cumulants of the standard Gaussian distribution are the mean, $0$, and the 
variance, $1$, and all subsequent Gaussian cumulants have value $0$.

\paragraph{Gaussian Mixture Model.}


For $i = 1, 2, \dots, m$, define Gaussian random vectors
$\eta_i \in \R^n$ with distribution $\eta_i \sim \Normal(\mu_i, \Sigma_i)$ where
$\mu_i \in \R^{n}$ and $\Sigma_i\in \R^{n\times n}$.  Let
$h$ be an integer-valued random variable which takes on value $i \in [m]$ with probability 
$w_i > 0$, henceforth called weights. (Hence $\sum_{i=1}^m w_i = 1$.)  Then, the random vector drawn as $Z = \eta_h$ is said
to be a Gaussian Mixture Model (GMM) 
$w_1 \Normal(\mu_1, \Sigma_1) + \ldots + w_m \Normal(\mu_m, \Sigma_m)$. 
The sampling of $Z$ can be interpreted as first picking
one of the components $i \in [m]$ according to the weights, and then sampling a Gaussian vector from 
component $i$.  We will be primarily interested in the mixture of
identical Gaussians of known covariance.  In particular, there exists known $\Sigma \in \R^{n\times n}$ such that $\Sigma_i = \Sigma$ for each $i$.
Letting $\eta \sim \Normal(0, \Sigma)$, and denoting by $\mathbf{e}_h$ the 
random variable which takes on the $i$\textsuperscript{th} canonical vector
$\mathbf{e}_i$ with probability $w_i$, we can write the GMM model as follows:
\begin{equation} \label{eq:GMM_Model}
  Z = [\mu_1 | \mu_2 | \cdots | \mu_m] \mathbf{e}_h + \eta \ .
\end{equation}
In this formulation, $\mathbf{e}_h$ acts as a selector of a Gaussian mean.  Conditioning on $h=i$, we have $Z \sim \Normal(\mu_i, \Sigma)$, which is consistent with the GMM model.

Given samples from the GMM, the goal is to recover the unknown parameters of the GMM, namely the means $\mu_1, \dots, \mu_m$ and the weights $w_1, \dots, w_m$.

\paragraph{Underdetermined ICA.}
In the basic formulation of ICA, the observed random variable $X \in \R^n$ is drawn according to the model $X = AS$, where $S \in \R^m$ is a latent random vector whose components $S_i$ are independent random variables, and $A \in \R^{n \times m}$ is an unknown \textit{mixing matrix}.
The probability distributions of the $S_i$ are unknown except that they are not Gaussian. 
The ICA problem is to
recover $A$ to the extent possible. The underdetermined ICA problem corresponds the case $m \geq n$.
We cannot hope to recover $A$ fully because if we flip the sign of the $i$\textsuperscript{th} column of $A$, or scale this column by some
nonzero factor, then the resulting mixing matrix with an appropriately scaled $S_i$ will again generate the same distribution on $X$ as before.  
There is an additional ambiguity that arises from not having an ordering on the coordinates $S_i$:
If $P$ is a permutation matrix, then $PS$ gives a new random vector with independent reordered coordinates, $AP^T$ gives a new mixing matrix with reordered columns, and $X = AP^TPS$ provides the same samples as $X = AS$ since $P^T$ is the inverse of $P$.
As $AP^T$ is a permutation of the columns of $A$, this ambiguity implies that we cannot recover the order of the columns of $A$.
However, it turns out that under certain genericity requirements, we can recover $A$ up to these necessary ambiguities, that is to say we can recover the directions (up to sign) of the columns of $A$, even in the underdetermined setting. 

In this paper, it will be important for us to work with an ICA model where there is Gaussian noise
in the data: 
$X = AS + \eta$,
where $\eta \sim \Normal(0, \Sigma)$ is an additive Gaussian noise independent of $S$, and the covariance of $\eta$ given by $\Sigma \in \R^{n \times n}$ is in general unknown and not necessarily spherical. 
We will refer to this model as the noisy ICA model.%

We define the flattening operation $\vec{\cdot}$ from a tensor to a vector in the natural way.  Namely, when and $T \in \R^{n^\ell}$ is a tensor, then $\vec{T}_{\delta(i_1, \dotsc, i_\ell)} = T_{i_1, \dotsc, i_\ell}$ where $\delta(i_1, \dotsc, i_\ell) = 1 + \sum_{j=1}^\ell n^{\ell-j} (i_j - 1)$ is a bijection with indices $i_j$ running from $1$ to $n$.
Roughly speaking, each index is being converted into a digit in a base $n$ number up to the final offset by 1.
This is the same flattening that occurs to go from a tensor outer product of vectors to the Kronecker product of vectors.

The ICA algorithm from \cite{GVX} to which we will be reducing learning a GMM relies on the shared tensor structure of the derivatives of the second characteristic function and the higher order multi-variate cumulants. This tensor structure motivates the following form of the Khatri-Rao product:
\begin{definition}
  Given matrices $A\in \R^{n_1 \times m}, B \in \R^{n_2\times m}$, a column-wise Khatri-Rao product is defined by $A \odot B := [\vec{A_1 \otimes B_1} | \cdots | \vec{A_m \otimes B_m}]$, where $A_i$ is the $i$\textsuperscript{th}
column of $A$, $B_i$ is the $i$\textsuperscript{th} column of $B$, $\otimes$ denotes the Kronecker product and 
$\vec{A_1 \otimes B_1}$ is flattening of the tensor $A_1 \otimes B_1$ into a vector.
  The related Khatri-Rao power is defined by $A^{\odot \ell} = A \odot \cdots \odot A$ ($\ell$ times).
\end{definition}
This form of the Khatri-Rao product arises when performing a change of coordinates under the ICA model using either higher order cumulants or higher order derivative tensors of the second characteristic function.

\paragraph{ICA Results.}
\vnote{We should have a brief discussion here on the typical sizes of $d$ and $k$.}
Theorem \ref{thm:UICA_noisy} (Appendix \ref{subsec:UICA_noisy}, from \cite{GVX}) allows us to recover $A$ up to the necessary ambiguities in the noisy ICA setting. 
The theorem establishes guarantees for an algorithm from \cite{GVX} for noisy underdetermined ICA, \textbf{UnderdeterminedICA}. This algorithm takes as input a tensor order parameter $d$, number of signals $m$, access to samples according to the noisy underdetermined ICA model with unknown noise, accuracy parameter $\epsilon$, confidence parameter $\delta$, bounds on moments and cumulants $M$ and $\Delta$, a bound on the conditioning parameter $\sigma_m$, and a bound on the cumulant order $k$. It returns approximations to the columns of $A$ up to sign and permutation. 

\section{Learning GMM means using underdetermined ICA: The basic idea}
\label{sec:reduction}
In this section we give an informal outline of the proof of our main result, namely learning the means of 
the components in GMMs via reduction to the underdetermined ICA problem.
Our reduction will be discussed in two parts. The first part gives the main idea of the reduction and will demonstrate how to recover the means $\mu_i$ up to their norms and signs, i.e. we will get $\pm \mu_i/\norm{\mu_i}$. We will then present the reduction in full. It combines the basic reduction with some preprocessing
of the data to recover the $\mu_i$'s themselves.
The reduction relies on some well-known properties of the Poisson distribution stated in the 
lemma below; its proof can be found in Appendix~\ref{sec:Poisson-lemmas}. 


\begin{lemma} \label{lem:Poisson-independence}
Fix a positive integer $k$, and let $p_i \geq 0$ be such that $p_1 + \dotsb + p_k =1$. If $X \sim \Pois(\lambda)$ and
$(Y_1, \ldots, Y_k)|_{X = x} \sim \Multinom(x; p_1, \ldots, p_k)$ then $Y_i \sim \Pois(p_i \lambda)$ for all $i$ and
$Y_1, \ldots, Y_k$ are mutually independent.
\end{lemma}

\paragraph{Basic Reduction: The main idea.}
Recall the GMM from equation \eqref{eq:GMM_Model} is given by 
$Z = [\mu_1 | \cdots | \mu_m] \mathbf{e}_h + \eta$. 
Henceforth, we will set $A = [\mu_1 | \cdots | \mu_m]$.  
We can write the GMM in the form $Z = A\mathbf{e}_h + \eta$, which is similar in form to the noisy ICA model, except that $\mathbf{e}_h$ does not have independent coordinates.
We now describe how a single sample of an approximate noisy ICA problem is generated.


The reduction involves two internal parameters $\lambda$ and $\tau$ that we will set later.
We generate a Poisson random variable $R \sim \Poisson(\lambda)$, 
and we run the following experiment $R$ times: At the $i$\textsuperscript{th} step, generate
sample $Z_i$ from the GMM.
Output the sum of the outcomes of these experiments: $Y = Z_1 + \cdots + Z_R$. 

Let $S_i$ be the random variable denoting the number of times samples were taken from the $i$\textsuperscript{th} Gaussian component in the above experiment.  Thus, $S_1 + \cdots + S_m = R$.
Note that $S_1, \dots, S_m$ are not observable although we know their sum. By 
Lemma~\ref{lem:Poisson-independence}, each $S_i$ has distribution $\Poisson(w_i \lambda)$, 
and the random variables $S_i$ are mutually independent.
Let $S := (S_1, \dots, S_m)^T$.

For a non-negative integer $t$, we define $\eta(t) := \sum_{i=1}^t \eta_i$ where the $\eta_i$ 
are iid according to $\eta_i \sim \Normal(0, \Sigma)$. In this definition, $t$ can be a random variable, in 
which case the $\eta_i$ are sampled independent of $t$. 
Using $\deq$ to indicate that two random variables have the same distribution, then 
$
  Y \deq AS + \eta(R)
$. 
If there were no Gaussian noise in the GMM (i.e.{}~if we were sampling from a discrete set of points) then the model becomes simply $Y = AS$, which is the ICA model without noise,
and so we could recover $A$ up to necessary ambiguities.
However, the model $Y \deq AS + \eta(R)$ fails to satisfy even the assumptions of the noisy ICA model, both because $\eta(R)$ is not independent of $S$ and because $\eta(R)$ is not distributed as a Gaussian random vector.

%

As the covariance of the additive Gaussian noise is known, we may add additional noise to the samples of $Y$ to obtain a good
approximation of the noisy ICA model.
Parameter $\tau$, the second parameter of the reduction, is chosen so that with high probability we have $R \leq \tau$.
Conditioning on the event $R \leq \tau$ we draw $X$ according to the rule $X = Y + \eta(\tau-R) \deq AS + \eta(R) + \eta(\tau-R)$,
where $\eta(R)$, $\eta(\tau - R)$, and $S$ are drawn independently conditioned on $R$.
Then, conditioned on $R \leq \tau$, we have $X \deq AS + \eta(\tau)$.

Note that we have only created an approximation to the ICA model.
In particular, restricting $\sum_{i=1}^m S_i = R \leq \tau$ can be accomplished using rejection sampling, but the coordinate random variables $S_1, \dots, S_m$ would no longer be independent.
We have two models of interest: (1) $X \deq AS + \eta(\tau)$, a noisy ICA model with no restriction on $R = \sum_{i=1}^m S_i$, and (2) $X \deq (AS + \eta(\tau))|_{R \leq \tau}$ the restricted model.

We are unable to produce samples from the first model, but it meets the assumptions of the noisy ICA problem.
Pretending we have samples from model (1), we can apply Theorem \ref{thm:UICA_noisy} (Appendix \ref{subsec:UICA_noisy}) to recover the Gaussian means up to sign and scaling.
On the other hand, we can produce samples from model (2), and depending on the choice of $\tau$, the statistical distance between models (1) and (2) can be made arbitrarily close to zero.
It will be demonstrated that given an appropriate choice of $\tau$, running \textbf{UnderdeterminedICA} on samples from model (2) is equivalent to running \textbf{UnderdeterminedICA} on samples from model (1) with high probability, allowing for recovery of the Gaussian mean directions $\pm \mu_i/\norm{\mu_i}$ up to some error.


\paragraph{Full reduction.} To be able to recover the $\mu_i$ without sign or scaling ambiguities, we add an extra coordinate to the GMM as follows. The new
means $\mu_i'$ are $\mu_i$ with an additional coordinate whose value is $1$ for all $i$, i.e.~%
$\mu_i' := \left( \mu_i^T, 1 \right)^T$.
Moreover, this coordinate has no noise. In
other words, each Gaussian component now has an $(n+1)\times (n+1)$ covariance matrix  
$\Sigma' := \left( \begin{smallmatrix} \Sigma & 0 \\ 0 & 0 \end{smallmatrix} \right)$.
It is easy to construct samples from this new GMM given samples from the original: If the original samples were
$u_1, u_2 \ldots$, then the new samples are $u'_1, u'_2 \ldots$ where
$u'_i  := \left( u_i^T, 1 \right)^T$.
The reduction proceeds similarly to the above on the new inputs.

Unlike before, we will define the ICA mixing matrix to be 
$A' := \bigl[ {\mu_1'}/{\norm{\mu_1'}} \bigl\lvert \cdots \bigr\rvert {\mu_m'}/{\norm{\mu_m'}}\bigr]$ such that it has unit norm columns.
The role of matrix $A$ in the basic reduction will now be played by $A'$.
Since we are normalizing the columns of $A'$, we have to
scale the ICA signal $S$ obtained in the basic reduction to compensate for this: Define $S'_i := \norm{\mu'_i} S_i$. Thus, the ICA models obtained in the full reduction are:
\begin{align}
X' &= A' S' + \eta'(\tau) \ ,  \label{eq:full-ideal-model} \\
X' &= (A' S' + \eta'(\tau))|_{R \leq \tau} \ , \label{eq:full-approximate-model}
\end{align} 
where we define $\eta'(\tau) = \left( \eta(\tau)^T, 0 \right)^T$.
As before, we have an ideal noisy ICA model \eqref{eq:full-ideal-model} from which we cannot sample, and an approximate noisy ICA model \eqref{eq:full-approximate-model} which can be made arbitrarily close to 
\eqref{eq:full-ideal-model} in statistical distance by choosing $\tau$ appropriately.  
With appropriate application of Theorem~\ref{thm:UICA_noisy} to these models, we can recover estimates (up to sign) $\{\tilde A_1', \dotsc, \tilde A_m'\}$ of the columns of $A'$.
%
%
%
%

By construction, the last coordinate of each $\tilde A_i'$ now
tells us both the sign and magnitude of each $\mu_i$: Let $\tilde{A}'_i(1:n) \in \R^n$ be the vector consisting of the
first $n$ coordinates of $\tilde{A}'_i$, and let $\tilde{A}'_i(n+1)$ be the last coordinate of $\tilde{A}'_i$. Then
$\mu_i = \frac{A_i'(1:n)}{A_i'(n+1)} \approx \frac{\tilde{A}'_i(1:n)}{\tilde{A}'_i(n+1)}, $
with the sign indeterminacy canceling in the division.

\floatname{algorithm}{Subroutine}
\begin{algorithm}[tb] 
\caption{Single sample reduction from GMM to approximate ICA} \label{sub:independentsamples}
\begin{algorithmic}[1]
\REQUIRE Covariance parameter $\Sigma$,
access to samples from a mixture of $\means$ identical Gaussians in $\R^{\dim}$ with variance $\Sigma$,
Poisson threshold $\tau$,
Poisson parameter $\lambda$,
\ENSURE $Y$ (a sample from model (\ref{eq:full-approximate-model})).
\vspace{.2em}
\hrule
\vspace{.2em}
  \STATE Generate $R$ according to $\Poisson(\lambda)$.
  \STATE If {$R > \tau$} return failure.
  \STATE Let $Y = 0$.
  \FOR{$j = 1$ to $R$}
    \STATE Get a sample $Z_j$ from the GMM.
    \STATE Let $Z'_j = (Z_j^T, 1)^T$ to embed the sample in $\R^{\dim+1}$.
    \STATE $Y = Y + Z_j'$. 
  \ENDFOR
  \STATE Let $\Sigma' = \begin{pmatrix} \Sigma & 0 \\ 0 & 0 \end{pmatrix}$ (add a row and column of all zeros)
  \STATE Generate $\eta'$ according to $\Norm(0,(\tau - R) \Sigma')$.
  \STATE $Y = Y + \eta'$.
\RETURN $Y$.
\end{algorithmic}
\end{algorithm}

\floatname{algorithm}{Algorithm}
\begin{algorithm}[tb]
\caption{Use ICA to learn the means of a GMM} \label{alg:reduction}
\begin{algorithmic}[1]
\REQUIRE Covariance matrix $\Sigma$, number of components $m$,
upper bound on tensor order parameter $d$,
access to samples from a mixture of $\means$ identical,
spherical Gaussians in $\R^{\dim}$ with covariance $\Sigma$,
confidence parameter $\delta$,
accuracy parameter $\epsilon$,
upper bound $w \geq \max_{i}(w_i) / \min_{i}(w_i)$,
upper bound on the norm of the mixture means $u$,
$r \geq (\max_i\norm{\mu_i} + 1)/(\min_i\norm{\mu_i})$, and
lower bound $b$ so $0 < b \leq \sigma_m(A^{\odot d/2})$.
\ENSURE $\{\tilde{\mu}_1, \tilde{\mu}_2, \dots, \tilde{\mu}_\means\} \subseteq \R^\dim$ (approximations to the means of the GMM).
\vspace{.2em}
\hrule
\vspace{.2em}
\STATE Let $\delta_2 = \delta_1 = \delta/2$.
\STATE Let $\sigma = \sup_{v \in S^{n-1}}\sqrt{\Var(v^T\eta(1))}$, for $\eta(1) \sim \Norm(0, \Sigma)$.
\STATE Let $\lambda = m$ be the parameter to be used to generate the Poisson random variable in Subroutine \ref{sub:independentsamples}.
\STATE Let $\tau = 4\big(\log(1/\delta_2) + \log(q(\Theta))\big)\max\left((e\lambda)^2, 4Cd^2\right)$
(the threshold used to add noise in the samples from Subroutine \ref{sub:independentsamples}, $C$ is a universal constant, and $q(\Theta)$ is a polynomial defined as (\ref{eq:q-theta}) in the proof of Theorem \ref{thm:correctness}).
\STATE Let $\epsilon^* = \epsilon \bigl(\sqrt{1 + u^2} + {2 (1 + u^2)}\bigr)^{-1}$.
\STATE Let $M = \max \bigl((\tau \sigma)^{d+1}, (w/(\sqrt{1+u^2})^{d+1}\bigr) (d+1)^{d+1}$.
\STATE Let $k = d + 1$.
\STATE Let $\Delta$ = $w$.
\STATE Invoke \textbf{UnderdeterminedICA} with access to Subroutine \ref{sub:independentsamples}, parameters $\delta_1, \epsilon^*$, $\Delta$, $M$, and $k$ to obtain $\tilde{A'}$ (whose columns approximate the normalized means up to sign and permutation).  If any calls to Subroutine \ref{sub:independentsamples} result in failure, the algorithm will halt completely.
\STATE Divide each column of $\tilde{A'}$ by the value of its last entry.
\STATE Remove the last row of $\tilde{A'}$ to obtain $\tilde{B}$.
\RETURN the columns of $\tilde{B}$ as $\{\tilde{\mu}_1, \tilde{\mu}_2, \dots, \tilde{\mu}_\means\}$.
\end{algorithmic}
\end{algorithm}


\section{Correctness of the Algorithm and Reduction}

Subroutine \ref{sub:independentsamples} captures the sampling process of the reduction:
Let $\Sigma$ be the covariance matrix of the GMM, $\lambda$ be an integer chosen as input, and a threshold value $\tau$ also computed elsewhere and provided as input.
Let $R \sim \Poisson(\lambda)$.
If $R$ is larger than $\tau$, the subroutine returns a failure notice and the calling algorithm halts immediately.
A requirement, then, should be that the threshold is chosen so that the chance of failure is very small; in our case, $\tau$ is chosen so that the chance of failure is half of the confidence parameter given to Algorithm \ref{alg:reduction}.
The subroutine then goes through the process described in the full reduction: sampling from the GMM, lifting the sample by appending a 1, then adding a lifted Gaussian so that the total noise has distribution $\Norm(0, \tau \Sigma)$.
The resulting sample is from the model given by (\ref{eq:full-approximate-model}).

Algorithm \ref{alg:reduction} works as follows: it takes as input the parameters of the GMM (covariance matrix, number of means), tensor order (as required by \textbf{UnderdeterminedICA}), error parameters, and bounds on certain properties of the weights and means.
The algorithm then calculates various internal parameters: a bound on directional covariances,  Poisson parameter $\lambda$, threshold parameter $\tau$, error parameters to be split between the ``Poissonization'' process and the call to \textbf{UnderdeterminedICA}, and values explicitly needed by \cite{GVX} for the analysis of \textbf{UnderdeterminedICA}.
Other internal values needed by the algorithm are denoted by the constant $C$ and polynomial $q(\Theta)$; their values are determined by the proof of Theorem \ref{thm:correctness}.
Briefly, $C$ is a constant so that one can cleanly compute a value of $\tau$ that will involve a polynomial, called $q(\Theta)$, of all the other parameters.
The algorithm then calls \textbf{UnderdeterminedICA}, but instead of giving samples from the GMM, it allows access to Subroutine \ref{sub:independentsamples}.
It is then up to \textbf{UnderdeterminedICA} to generate samples as needed (bounded by the polynomial in Theorem \ref{thm:correctness}).
In the case that Subroutine \ref{sub:independentsamples} returns a failure, the entire algorithm process halts, and returns nothing.
If no failure occurs, the matrix returned by \textbf{UnderdeterminedICA} will be the matrix of normalized means embedded in $\R^{\dim+1}$, and the algorithm de-normalizes, removes the last row, and then has approximations to the means of of the GMM.

The bounds are used instead of actual values to allow flexibility --- in the context under which the algorithm is invoked --- on what the algorithm needs to succeed.
However, the closer the bounds are to the actual values, the more efficient the algorithm will be.


\paragraph{Sketch of the correctness argument.}
The proof of correctness of Algorithm \ref{alg:reduction} has two main parts.
For brevity, the details can be found in Appendix \ref{sec:mainproof}.
In the first part, we analyze the sample complexity of recovering the Gaussian means using \textbf{UnderdeterminedICA} when samples are taken from the ideal noisy ICA model \eqref{eq:full-ideal-model}.

In the second part, we note that we do not have access to the ideal model \eqref{eq:full-ideal-model}, and that we can only sample from the approximate noisy ICA model \eqref{eq:full-approximate-model} using the full reduction.
Choosing $\tau$ appropriately, we use total variation distance to argue that with high probability, running \textbf{UnderdeterminedICA} with samples from the approximate noisy ICA model will produce equally valid results as running \textbf{UnderdeterminedICA} with samples from the ideal noisy ICA model.
The total variation distance bound is explored in section \ref{subsec:TotalVarDist}.

These ideas are combined in section \ref{subsec:CorrectnessProof} to prove the correctness of Algorithm \ref{alg:reduction}.
One additional technicality arises from the implementation of Algorithm \ref{alg:reduction}.
Samples can be drawn from the noisy ICA model $X' = (AS' + \eta'(\tau))|_{R \leq \tau}$ using rejection sampling on $R$.
In order to guarantee Algorithm \ref{alg:reduction} executes in polynomial time, when a sample of $R$ needs to be rejected, Algorithm \ref{alg:reduction} terminates in explicit failure.
To complete the proof, we argue that with high probability, Algorithm \ref{alg:reduction} does not explicitly fail.

\section{Smoothed Analysis} \label{sec:smoothed}
We start with a base matrix $M \in \R^{n \times \binom{n}{2}}$ and add a perturbation
matrix $N \in \R^{n \times \binom{n}{2}}$ with each entry coming iid from $\Normal(0, \sigma^2)$ for some 
$\sigma > 0$. [We restrict the discussion to the second power for simplicity; extension to higher power is
straightforward.] As in \cite{GVX}, it will be convenient to work with the multilinear part of the Khatri--Rao 
product: For a column vector $A_k \in \R^n$ define $A_k^{\ominus 2} \in \R^{\binom{n}{2}}$, a subvector of 
$A_k^{\odot 2} \in \R^{n^2}$, given by $(A_k^{\ominus 2})_{ij} := (A_k)_i (A_k)_j$ for $1 \leq i < j \leq n$. Then 
for a matrix $A = [A_1, \ldots, A_m]$ we have $A^{\ominus 2} := [A_1^{\ominus 2}, \ldots, A_m^{\ominus 2}]$. 

\begin{theorem} \label{thm:smoothed-sigma-min-multilinear}
With the above notation, for any base matrix $M$ with dimensions as above, we have, for some absolute constant $C$,
\begin{align*}
\prob{\sigma_{\min}((M+N)^{\ominus 2}) \leq \frac{\sigma^2}{n^7}} \leq \frac{2C}{n}.
\end{align*}
\end{theorem}

Theorem~\ref{thm:smoothed-sigma-min-intro} follows immediately from the theorem above by noting
that $\sigma_{\min}(A^{\odot 2}) \geq \sigma_{\min}(A^{\ominus 2})$. 

\begin{proof}
In the following, for a vector space $V$ (over the reals) $\dist(v, V')$ denotes the distance between vector
$v \in V$ and subspace $V' \subseteq V$; more precisely, $\dist(v, V') := \min_{v' \in V'} \norm{v-v'}_2$.
We will use a lower bound on $\sigma_{\min}(A)$, found in Appendix \ref{subsec:rudelson-vershynin}.

With probability $1$, the columns of the matrix $(M+N)^{\ominus 2}$ are linearly 
independent. This can be proved along the lines of a similar result in \cite{GVX}. 
Fix $k \in {\binom{n}{2}}$ and let $u \in \R^{{\binom{n}{2}}}$ be a 
unit vector orthogonal to the subspace spanned by the columns of $(M+N)^{\ominus 2}$ other than column $k$. Vector
$u$ is well-defined with probability $1$. Then the distance of the $k$'th column $C_k$ from the span of the 
rest of the columns is given by 
\begin{align}
u^T C_k &= u^T (M_k+N_k)^{\ominus 2} \nonumber = \sum_{1 \leq i < j \leq n} u_{ij} (M_{ik}+N_{ik})(M_{jk}+N_{jk}) \nonumber \\
&=  \sum_{1 \leq i < j \leq n} u_{ij} M_{ik}M_{jk} + \sum_{1 \leq i < j \leq n} u_{ij}M_{ik}N_{jk} 
+ \sum_{1 \leq i < j \leq n} u_{ij}N_{ik}M_{jk} + \sum_{1 \leq i < j \leq n} u_{ij}N_{ik}N_{jk} \nonumber \\
&=: P(N_{1k}, \ldots, N_{nk}). \label{eqn:polynomial}
\end{align}

Now note that this is a quadratic polynomial in the random variables $N_{ik}$. We will apply the anticoncentration
inequality of \processifversion{vstd}{Carbery--Wright~}\cite{CarberyWright} to this polynomial to conclude that the distance between 
the $k$'th column of $(M+N)^{\ominus 2}$ and the span of the rest of the columns is unlikely to be very small (see Appendix \ref{subsec:carbery-wright} for the precise result).

Using $\norm{u}_2 = 1$, the variance of our polynomial in \eqref{eqn:polynomial} becomes 
\begin{align*}
\Vr{P(N_{1k}, \ldots, N_{nk})} 
= \sigma^2 \biggl(\sum_j \Bigl(\sum_{i:i < j} u_{ij}M_{ik}\Bigr)^2 + \sum_i \Bigl(\sum_{j: i < j} u_{ij}M_{jk}\Bigr)^2 \biggr) 
+ \sigma^4 \sum_{i < j} u_{ij}^2  \geq \sigma^4.
\end{align*}
In our application, our random variables $N_{ik}$ for $i \in [n]$ are not standard
Gaussians but are iid Gaussian with variance $\sigma^2$, and our polynomial does not have unit variance. 
After adjusting for these differences using 
the estimate on the variance of $P$ above, Lemma~\ref{lem:CarberyWright} gives 
$
\prob{\abs{P(N_{1k}, \ldots, N_{nk})-t} \leq \epsilon} \leq 
2C \sqrt{{\epsilon}/{\sigma^2}} = 2C \sqrt{\epsilon}/\sigma$.


Therefore, by the union bound over the choice of $k$
$\prob{\text{there is a } k \text{ such that } \dist(C_k, C_{-k}) \leq \epsilon} 
\leq {\binom{n}{2}} 2C  \sqrt{\epsilon}/\sigma$.

Now choosing $\epsilon = \sigma^2/n^6$, Lemma~\ref{lem:RudelsonVershynin} gives
$\prob{\sigma_{\min}((M+N)^{\ominus 2}) \leq {\sigma^2}/{n^7}} \leq {2C}/{n}$. 
\end{proof}

We note that while the above discussion is restricted to Gaussian perturbation, the same technique would work
for a much larger class of perturbations. To this end, we would require a version of the Carbery-Wright
anticoncentration inequality which is applicable in more general situations. We omit such generalizations here.


\section{The curse of low dimensionality for Gaussian mixtures}
\label{identifiability_low_dimension}
In this section we prove Theorem \ref{thm:low-dim-identifiability}, which informally says that for small $n$ there is a large class of superpolynomially close mixtures in $\R^n$ with fixed variance.  This goes beyond the specific example of exponential closeness  given in~\cite{moitra10} as we demonstrate that such mixtures are ubiquitous as long as there is no lower bound on the separation between the components.

Specifically, let $S$ be the cube $[0,1]^n \subset \R^n$. 
We will show that for any two sets of $k$ points $X$ and $Y$ in $S$, with fill $h$ (we say that $X$ has fill $h$, if there is a point of $X$ within distance $h$ of any point of $S$), there exist two mixtures $p,q$ with means on disjoint subsets of $X\cup Y$, which are exponentially close in $1/h$ in the $L^1(\R^n)$ norm. 
Note that the fill of a sample from the uniform distribution on the cube can be bounded (with high probability) by $O(\frac{\log k}{k^{1/n}})$.\lnote{fix}

We start by defining some of the key objects.  Let $K(x,z) = (2\pi)^{-n/2} e^{-\norms{x-y}/2}$ be the unit Gaussian kernel.
Let ${\cal K}$ be the integral operator corresponding to the convolution with a unit Gaussian: ${\cal K}g(z)= \int_{\R^n} K(x,z)g(x) dx$.
Let $X$ be any subset of $k$ points in $[0,1]^n$.
Let $K_X$ be the kernel matrix corresponding to $X$, $(K_X)_{ij} = K(x_i,x_j)$. 
It is known to be positive definite.
For a function $f: [0,1]^n \to \R$, the \emph{interpolant} is defined as $f_{X,k}(x) = \sum w_i K(x_i,x)$, where the coefficients $w_i$ are chosen so that  $(\forall i) f_{X,k}(x_i)=f(x_i)$. 
It is easy to see that such interpolant exists and is unique, obtained by solving a linear system involving $K_X$.

We will  need some properties of the Reproducing Kernel Hilbert Space $H$ corresponding to the kernel $K$ (see \cite[Chapter 10]{wendland2005scattered} for an introduction).
In particular, we need the bound $\|f\|_\infty \le \|f\|_H$ and  the reproducing property, $\langle f(\cdot), K(x,\cdot)\rangle_H = f(x), \forall f \in H$. For a function of the form $\sum w_i K(x_i, x)$ we have $\norms{\sum w_i K(x_i, x)}_H = \sum w_i w_j K(x_i, x_j) $.

\begin{lemma}\label{lem:linfinity}
Let $g$ be any positive function with $L_2$ norm $1$ supported on $[0,1]^n$ and let  $f={\cal K}g$.
If $X$ has fill $h$, then there exists $A>0$ such that
$$
\| f - f_{X,k}\|_{L^\infty(\R^n)}  < \exp (A \frac{\log h}{h}).
$$ 
\end{lemma}
\begin{proof}
From~\cite{rieger2010sampling}, Theorem 6.1 (taking $\lambda=0$)\details{it seems to be already in the book for $\lambda = 0$} we have that for some $A>0$ and $h$ sufficiently small
$
\|f - f_{X,k}\|_{L^2([0,1]^n)} < \exp (A \frac{\log h}{h}).
$
Note that the norm is on $[0,1]^n$ while we need to control the norm on $\R^n$.
To do that we need a bound on the RKHS norm of $f - f_{X,k}$. This ultimately gives control of the norm over $\R^n$ because there is a canonical isometric embedding of elements of $H$ interpreted as functions over $[0,1]$ into elements of $H$ interpreted as functions over $\R^n$.
We first observe that for any $x_i \in X$, $f(x_i) - f_{X,k}(x_i) = 0$. Thus, from the reproducing property of RKHS, $\langle f - f_{X,k}, f_{X,k}\rangle_H = 0$. 
Using properties of RKHS with respect to the operator ${\cal K}$ (see, e.g., Proposition 10.28 of ~\cite{wendland2005scattered})
\begin{align*}
\|f - f_{X,k}\|_H^2 
&= \langle f - f_{X,k}, f - f_{X,k}\rangle_H 
= \langle f - f_{X,k}, f \rangle_H  
= \langle f - f_{X,k}, {\cal K}g\rangle_H\\ 
&= \langle f - f_{X,k}, g \rangle_{L_2([0,1]^n)} 
\le  \|f - f_{X,k}\|_{L^2([0,1]^n)} \|g\|_{L^2([0,1]^n)} 
< \exp (A \frac{\log h}{h}).
\end{align*}
Thus
$
\|f - f_{X,k}\|_{L^\infty(\R^n)} \le \|f - f_{X,k}\|_H < \exp (A \frac{\log h}{h}).
$
\end{proof}

\begin{theorem}\label{thm:differentmixtures}
Let $X$ and $Y$ be any two subsets of $[0,1]^n$ with fill $h$. 
Then there exist two Gaussian mixtures $p$ and $q$ (with positive coefficients summing to one, but not necessarily the same number of components), which are centered on
two disjoint subsets of $X\cup Y$ and such that\details{remember that $h<1$} for some $B>0$
$$
\|p-q\|_{L^1(\R^n)} < \exp(B \frac{\log h}{h}).
$$

\end{theorem}
\begin{proof}
To simplify the notation we assume that $n=1$.  The general case follows verbatim, except that the interval 
of integration, $[-1/h,1/h]$, and its complement need to be replaced by the sphere of radius $1/h$ and its complement respectively.

Let  $f_{X,k}$ and $f_{Y,k}$ be the interpolants, for some fixed sufficiently smooth (as above, $f =\mathcal{K} g$) positive function $f$ with $\int_{[0,1]} f(x) dx = 1$. 
Using Lemma \ref{lem:linfinity}, we see that $\|f_{X,k} - f_{Y,k}\|_{L^\infty(\R)} < 2\exp (A \frac{\log h}{h}) $. \details{pruning?}
Functions $f_{X,k}$ and $f_{Y,k}$ are both linear combinations of Gaussians possibly with negative coefficients and so is $f_{X,k} -  f_{Y,k}$ .  By collecting positive and negative coefficients  we write  
\begin{equation}\label{equ:shuffling}
f_{X,k} -  f_{Y,k} = p_1 - p_2, 
\end{equation}
where, $p_1$ and $p_2$ are mixtures with positive coefficients only.

Put $p_1 = \sum_{i \in S_1} \alpha_i K(x_i, x)$, $p_2 = \sum_{i \in S_2} \beta_i K(x_i, x)$, where $S_1$ and $S_2$ are disjoint subsets of $X \cup Y$. Now we need to ensure that the coefficients can be normalized to sum to $1$.

Let $\alpha = \sum \alpha_i$, $\beta = \sum \beta_i$. 
From \eqref{equ:shuffling} and by integrating over the interval $[0,1]$, and since $f$ is strictly positive on the interval,  it is easy to see that $\alpha, \beta \geq 1$.
\details{Use $\int_{[0,1]} K(x_i, x) < c < 1$ and $\int_{[0,1]} \alpha_i K(x_i, x) \geq \int_{[0,1]} f_{X,k}$.} 
We have
$$
|\alpha - \beta| = \left | \int_{\R} p_1(x) - p_2(x) dx \right| \le \|p_1 - p_2\|_{L^1(\R)} 
$$
$$
\|p_1 - p_2\|_{L^1(\R)}  \le \int_{[-1/h,1/h]} \|f_{X,k} - f_{Y,k}\|_{L^\infty(\R)} dx + 2 (\alpha + \beta) \int_{x \in [1/h,\infty)} K(0,x-1)dx.
$$

Noticing that the first summand is bounded by $\frac{2}{h} \exp (A \frac{\log h}{h})$ and the integral in the second summand is even smaller (in fact, $O(e^{-1/h^2})$) \details{at most $\frac{e^{-1/h^2}}{(1/h) - 1}$}, 
it follows immediately, that $|1 - \frac{\beta}{\alpha}| < \exp(A' \frac{\log h}{h})$ for some $A'$ and $h$ sufficiently small.

Hence,  we have
$
\left \|\frac{1}{\alpha} p_1 - \frac{1}{\beta}p_2 \right\|_{L^{1}(\R)} \le 
\left \|\frac{\beta}{\alpha} p_1 -  p_2 \right\|_{L^{1}(\R)} \le \left |1 - \frac{\beta}{\alpha}\right| \|p_1\|_{L^1(\R)}+
\left \| p_1 -  p_2 \right\|_{L^{1}(\R)}. 
$

Collecting exponential inequalities completes the proof.
\end{proof}

\begin{proof}[of Theorem \ref{thm:low-dim-identifiability}]
For convenience we will use a set of $4k^2$ points instead of $k^2$. Clearly it does not affect the exponential rate. 

By a simple covering set argument (cutting the cube into $m^n$ cubes with size $1/m$) and basic probability, we see that the fill $h$ of $n m^n \log m$ points is at most $O(\sqrt{n}/m)$ with probability $1-o(1)$. 
Hence, given $k$ points, we have $h = O((\frac{\log k}{k})^{1/n})$. 
We see, that with a smaller probability (but still close to $1$ for large $k$), 
 we can sample $k$ points $3k^2$ times and still have the same fill.

Partitioning the set of $4k^2$ points into $2k$ disjoint subsets of $2k$ points and applying Theorem \ref{thm:differentmixtures} (to $k+k$ points) we obtain 
$2k$ pairs of exponentially close mixtures with at most $2k$ components each. 
If one of the pairs has the same number of components, we are done. If not, 
by the pigeon-hole principle for at least two pairs of mixtures $p_1 \approx q_1$ and $p_2 \approx q_2$ the differences of the number of components (an integer  number between $0$ and $2k-1$) must coincide.  
Assume without loss of generality that $p_1$ has no more components that $q_1$ and $p_2$ has no more components than $q_2$.Taking $p = \frac{1}{2}(p_1 + q_2)$ and $q= \frac{1}{2}(p_2 + q_1)$ completes the proof.
\end{proof}



\begin{vstd}
\bibliographystyle{abbrv}
\end{vstd}
\begin{vCOLT}
\newpage
\end{vCOLT}
\bibliography{GMM_bibliography}

\appendix

\section{Theorem~\ref{thm:correctness} Proof Details}\label{sec:mainproof}

\subsection{Error Analysis of the Ideal Noisy ICA Model}
\label{subsec:IdealErrorAnalysis}

The proposed full reduction from Section \ref{sec:reduction} provides us with two models.
The first is a noisy ICA model from which we cannot sample:
\begin{align}
\text{(Ideal ICA)} \qquad  X' &= A'S' + \eta'(\tau) \ . \label{A'_ICAModel_ideal}
\end{align}
The second is a model that fails to satisfy the assumption that $S'$ has independent coordinates, but it is a model from which we can sample:
\begin{align}
\text{(Approximate ICA)} \qquad  X' &= (A'S' + \eta'(\tau))|_{R \leq \tau} \ . \label{A'_ICAModel_actual}
\end{align}
Both models rely on the choice of two parameters, $\lambda$ and $\tau$.
The dependence on $\tau$ is explicit in the models.  
The dependence on $\lambda$ can be summarized in the unrestricted model as $S_i = \frac 1 {\norm{\mu_i'}} S_i' \sim \Poisson(w_i \lambda)$ independently of each other, and $R = \sum_{i=1}^m S_i \sim \Poisson(\lambda)$.

The probability of choosing $R > \tau$ will be
seen to be exponentially small in $\tau$.  For this reason, running \textbf{UnderdeterminedICA}
with polynomially many samples from model
\eqref{A'_ICAModel_ideal} will with high probability be equivalent to
running the ICA Algorithm with samples from model
\eqref{A'_ICAModel_actual}.  This notion will be made
precise later using total variation distance.


For the remainder of this subsection, we proceed as if samples are drawn from
the ideal noisy ICA model \eqref{A'_ICAModel_ideal}.  Thus, to recover the columns of $A'$, it suffices to run \textbf{UnderdeterminedICA} on samples of $X'$.
Theorem \ref{thm:UICA_noisy} can be used for this analysis so long as we can
obtain the necessary bounds on the cumulants of $S'$, moments of $S'$, and the moments of $\eta'(\tau)$.
We define $w_{\min} := \min_{i} w_i$ and $w_{\max} := \max_{i} w_i$.
Then, the cumulants of $S'$ are bounded by the following lemma:

\begin{lemma} \label{lem:cumS_i'}
  Given $\ell \in \Z^+$, $\cum_\ell(S_i') \geq w_i \lambda$ for each $S_i'$.
  In particular, then $\cum_\ell(S_i') \geq w_{\min} \lambda$.
\end{lemma}
\begin{proof}
  By construction, $S'_i = \norm{\mu_i'}S_i$.
  By the homogeneity property of univariate cumulants,
  \[
    \cum_\ell(S'_i) = \cum_\ell(\norm{\mu_i'}S_i) 
      = \norm{\mu_i'}^\ell \cum_\ell(S_i)
  \]
  As $\mu_i'(n+1) = 1$, $\norm{\mu_i'} \geq 1$. \details{Aha! This is one of the steps that would kill invariance under scaling of the algorithm.}
  The cumulants of the Poisson distribution are given in Lemma \ref{lem:PoissonCumulants}.  It follows that $\cum_{\ell}(S_i') \geq \cum_{\ell}(S_i) = w_i \lambda$.
\end{proof}

The bounds on the moments of $S'_i$ for each $i$ can be computed using the following lemma:
\begin{lemma} \label{lem:expectation_S_i'}
For $\ell \in \Z^+$, we have $\E{S_i'^\ell} \leq (\norm{\mu_i'}w_i\lambda)^\ell \ell^\ell$.
\end{lemma}
\begin{proof}
  Let $Y$ denote a random variable drawn from $\Poisson(\alpha)$.
  It is known (see \cite{Riordan}) that
  \[
    \E{Y^\ell} = \sum_{i=1}^\ell \alpha^i \stirling{\ell}{i}
  \]
  where $\stirling{\ell}{i}$ denotes Stirling number of the second kind.
  Using Lemma \ref{lem:StirlingBound_n}, it follows that
  \[
    \E{Y^\ell} \leq \sum_{i=1}^\ell \alpha^{i} \ell^{\ell-1}
      \leq \ell \alpha^{\ell} \ell^{\ell-1} 
      = \alpha^\ell\ell^\ell. 
  \]
  Since $S_i' = \mu_i'S_i$ where $S_i \sim \Poisson(\lambda w_i)$, it follows that $\E{S_i'^\ell} = \norm{\mu_i'}^\ell\E{S_i^\ell} \leq \norm{\mu_i'}^\ell (w_i \lambda_i)^\ell \ell^\ell$.
\end{proof}

The absolute moments of Gaussian random variables are well known.  For
completeness, the bounds are provided in Lemma \ref{lem:GaussianMoments} of
Appendix \ref{sec:GaussMoments}.

Defining $\sigma = \sup_{v \in S^{n-1}}\sqrt{\Var(v^T\eta'(1))}$; vectors $\mu_{\max}' = \max_{i}\norm{\mu_i'}$, $\mu_{\min}' = \min_{i}\norm{\mu_i'}$, and similarly $\mu_{\max}$ and $\mu_{\min}$ for later; and choosing $\lambda = m$, we can now show a polynomial bound for the error in recovering the columns of $A'$
using \textbf{UnderdeterminedICA}.
\begin{theorem}[ICA specialized to the ideal case] \label{thm:ICA_ideal_lifted}
  Suppose that samples of $X'$ are taken from the 
  unrestricted ICA model \eqref{eq:full-approximate-model} choosing parameter $\lambda = m$ and $\tau$ a constant.
  Suppose that \textbf{UnderdeterminedICA} is run using these samples.
  Suppose $\sigma_m(A'^{\odot d/2}) > 0$.  Fix $\epsilon \in (0, 1/2)$ and $\delta \in (0, 1/2)$. 
  Then with probability $1-\delta$, when the number of samples $N$ is:
   \begin{align}
    N & \geq \poly\left(n^d, m^{d^2}, 
      (\tau\sigma)^{d^2}, \norm{\mu'_{\max}}^{d^2}, \left({w_{\max}}/{w_{\min}}\right)^{d^2},
      d^{d^2}, 
      1/\sigma_m(A'^{\odot d/2})^d, 1/\epsilon, 1/\delta \right) \label{eq:poly_bound_lifted_recovery}
  \end{align}
  the columns of $A'$ are recovered within error $\epsilon$ up to their signs.
  That is, denoting the
  columns returned from \textbf{UnderdeterminedICA} by $\tilde A_1', \dotsc, \tilde A_m'$, 
  there exists $\alpha_1, \cdots, \alpha_m,  \in \{-1, +1\}$ and a permutation $p$ of $[m]$ such that
    $\norm{A'_i - \alpha_i\tilde A'_{p(i)}} < \epsilon$
  for each $i$\lnote{prime/no prime?}.
\end{theorem}
\begin{proof}
   Obtaining the sample bound is an exercise of rewriting the parameters associated with the model $X' = A'S' + \eta'(\tau)$ in a way which can be used by Theorem \ref{thm:UICA_noisy}.
  In what follows, where new parameters are introduced without being described, they will correspond to parameters of the same name defined in and used by the statement of Theorem \ref{thm:UICA_noisy}.
  
  Parameter $d$ is fixed.  We must choose $k_1, \dotsc, k_m$ and $k$ such that $d < k_i \leq k$ and $\cum_{k_i}(S_i')$ is bounded away from $0$.
  It suffices to choose $k_1 = \dotsb = k_m = k = d+1$.
  By Lemma \ref{lem:cumS_i'}, $\cum_{d+1}(S_i') \geq w_{\min} \lambda = w_{\min} m$ for each $i$.  As 
  $w_{\max} \geq \frac 1 m \sum_{i=1}^m w_i = \frac 1 m$, we have that $\cum_{d+1}(S_i') \geq \frac {w_{\min}}{w_{\max}}$ for each $i$, giving a somewhat more natural condition number.  
  In the notation of Theorem \ref{thm:UICA_noisy}, we have a constant
  \begin{equation} \label{eq:DeltaChoice}
    \Delta = \frac{w_{\min}}{w_{\max}}
  \end{equation}
  such that $\cum_{d+1}(S_i') \geq \Delta$ for each $i$.
  
  Now we consider the upper bound $M$ on the absolute moments of both 
  $\eta'(\tau)$ and on $S_i'$.  As the Poisson distribution takes on non-negative
  values, it follows that $S_i' = \norm{\mu_i'}S_i$ takes on non-negative values.
  Thus, the moments and absolute moments of $S_i'$ coincide.
  Using Lemma \ref{lem:expectation_S_i'}, we have that $\E{\abs{S_i'}^{d+1}}
  = \E{(S_i')^{d+1}} \leq (\norm{\mu_i'}w_i\lambda)^{d+1} (d+1)^{d+1}$.  Thus,
  for $M$ to bound the $(d+1)$\textsuperscript{th} moment of $S_i'$, it suffices
  that $M \geq (\norm{\mu_{\max}'}w_{\max} \lambda)^{d+1} (d+1)^{d+1}$.
  Noting that 
  \[
    w_{\max} \lambda = w_{\max} m = \frac{w_{\max}}{1/m} \leq \frac{w_{\max}}{w_{\min}}
  \]
  it suffices that $M \geq (\norm{\mu_{\max}'}\frac{w_{\max}}{w_{\min}})^{d+1} (d+1)^{d+1}$,
  giving a more natural condition number.
  
  Now we bound the absolute moments of the Gaussian distribution.  
  As $d \in 2\N$, it follows that $d+1$ is odd.
  Given a unit vector $u \in \R^n$, it follows from Lemma \ref{lem:GaussianMoments} that 
  \[
    \E{\abs{\langle u, \eta'(\tau) \rangle}^{d+1}}
      = \Var(\langle u, \eta'(\tau)\rangle)^{\frac {(d+1)} 2} 2^{d/2}\left(d / 2\right) ! \frac 1 {\sqrt \pi} = \tau^{d+1}\Var(\langle u, \eta'(1) \rangle)^{\frac{(d+1)}2} 2^{d/2}\left(d / 2\right) ! \frac 1 {\sqrt \pi} \ .
  \]
  $\sigma$ gives a clear upper bound for $\Var(\langle u, \eta'(1) \rangle)^{1/2}$, and $(d+1)^{d+1}$ 
  gives a clear upper bound to $\frac 1 {\sqrt \pi} 2^{d/2}(d / 2)!$. 
  As such, it suffices that $M \geq (\tau \sigma)^{d+1}(d+1)^{d+1}$ in order to
  guarantee that $M \geq \E{\abs{\langle u, \eta'(\tau) \rangle}^{d+1}}$.
  Using the obtained bounds for $M$ from the Poisson and Normal variables,
  it suffices that $M$ be taken such that 
  \begin{equation} \label{eq:MChoice}
        M \geq \max \left((\tau \sigma)^{d+1}, (\norm{\mu_{\max}'} \frac{w_{\max}}{w_{\min}})^{d+1}\right) (d+1)^{d+1}
  \end{equation}
  to guarantee that $M$ bounds all required order $d+1$ absolute moments.
  
  We can now apply Theorem \ref{thm:UICA_noisy}, using the parameter values $k = d+1$,
  $\Delta$ from \eqref{eq:DeltaChoice}, and $M$ from \eqref{eq:MChoice}.
  Then with probability $1 - \delta$,
  \begin{align}
    N &\geq \poly\left(n^{2d+1}, m^{d^2}, 
      (\tau \sigma)^{d^2}, \norm{\mu_{\max}'}^{d^2}, ({w_{\max}}/{w_{\min}})^{d^2},
      (d+1)^{d^2}, \right. \notag \\
      & \quad \quad \quad \quad \left.
      1/\sigma_m(A'^{\odot d/2})^{d+1}, 1/\epsilon, 1/\delta \right) \label{eq:N_poly_with_mu_prime}
  \end{align} 
  samples suffice to recover up to sign the columns of $A'$ within $\epsilon$ accuracy.  
  More precisely, letting 
  $\tilde A_1', \dotsc, \tilde A_m'$ give the columns produced by 
   \textbf{UnderdeterminedICA}, then there exists parameters $\alpha_1, \dotsc, \alpha_m$ such that $\alpha_i \in \{-1, +1\}$ captures the sign 
  indeterminacy, and a permutation $p$ on $[m]$ such that $\norm{A_i' - \tilde A_{p(i)}'} < \epsilon$ for each $i$.
  
  The poly bound in \eqref{eq:N_poly_with_mu_prime} is equivalent to the poly bound in \eqref{eq:poly_bound_lifted_recovery}.
\end{proof}

Theorem \ref{thm:ICA_ideal_lifted} allows us to recover the columns of $A'$ up
to sign.  However, what we really want to recover are the means of the original
Gaussian mixture model, which are the columns of $A$.  Recalling the
correspondence between $A'$ and $A$ laid out in section \ref{sec:reduction},
the Gaussian means $\mu_1, \dotsc, \mu_m$ which form the columns of $A$ are
related to the columns $\mu_1', \dotsc, \mu_m'$ of $A'$ by the rule 
$\mu_i = \mu_i'(1:n) / \mu_i'(n+1)$.  Using this rule, we can construct estimate
the Gaussian means from the estimates of the columns of $A'$.  By propagating
the errors from Theorem \ref{thm:ICA_ideal_lifted}, we arrive at the following
result:

\begin{theorem}[Recovery of Gaussian means in Ideal Case] \label{thm:MeanRecoveryIdeal}
  Suppose that \textbf{UnderdeterminedICA} is run using samples of $X'$ from the
  ideal noisy ICA model \eqref{A'_ICAModel_ideal} choosing parameters $\lambda = m$ and $\tau$ a constant.
  Define $B\in \R^{n \times m}$ such that $B_i = A_i / \norm{A_i}$.  
  Suppose further that $\sigma_m(B^{\odot d/2}) > 0$.  Let 
  $\tilde A_1', \cdots, \tilde A_m'$ be the returned estimates of the columns of $A'$ (from model \eqref{A'_ICAModel_ideal}) by \textbf{UnderdeterminedICA}.
  Let $\tilde \mu_i = \tilde A_i'(1:n) / \tilde A_i'(n+1)$ for each $i$.
  Fix error parameters $\epsilon \in (0, 1/2)$ and $\delta \in (0, 1/2)$.
  When at least
  \begin{align}
    N & \geq \poly\left(n^{d}, m^{d^2}, 
      (\tau \sigma)^{d^2}, \norm{\mu_{\max}}^{d^2}, \left(\frac{w_{\max}}{w_{\min}}\right)^{d^2},
      d^{d^2}, 
      \left(\frac{\norm{\mu_{\max}} + 1}{\norm{\mu_{\min}}}\right)^{d^2}, \frac 1 {\sigma_m(B^{\odot d/2})^{d}}, \frac 1 \epsilon, \frac 1 \delta \right)
    \label{eq:poly_bound_ideal_means}
  \end{align}
  samples are used, then with probability $1 - \delta$ there exists a permutation
  $p$ of $[m]$ such that $\norm{\tilde \mu_{p(i)} - \mu_i} < \epsilon$ for each $i$.
\end{theorem}
\begin{proof}
  Let $\epsilon^* > 0$ (to be chosen later) give a desired bound on the 
  errors of the columns of $A'$.  Then, from Theorem \ref{thm:ICA_ideal_lifted},
  using
  \begin{align}
    N &\geq \poly\left(n^{d}, m^{d^2}, 
      (\tau\sigma)^{d^2}, \norm{\mu_{\max}'}^{d^2}, \left({w_{\max}}/{w_{\min}}\right)^{d^2},
      d^{d^2}, 
      1/\sigma_m(A'^{\odot d/2})^{d}, 1/\epsilon^*, 1/\delta \right) \label{eq:sample_bound_for_err_prop0}
  \end{align}
  samples suffices with probability $1 - \delta$ to produce column estimates $\tilde A_1', \dotsc, \tilde A_m'$ such
  that for an unknown permutation $p$ and signs $\alpha_1, \dotsc, \alpha_m$,
  $\alpha_{p(1)}\tilde A_{p(1)}', \dotsc, \alpha_{p(m)}\tilde A_{p(m)}'$ give
  $\epsilon^*$-close estimates of the columns $A_1', \dotsc, A_m'$ 
  respectively of $A'$.
  In order to avoid notational clutter, we will assume without
  loss of generality that $p$ is the identity map, and hence that 
  $\norm{\alpha_i\tilde A_i' - \alpha A_i'} < \epsilon^*$ holds.
  
  This proof proceeds in two steps.  
  First, we replace the dependencies in \eqref{eq:sample_bound_for_err_prop0} on parameters from the lifted GMM model generated by the full reduction with dependencies based on the GMM model we are trying to learn.
  Then, we propagate the error from recovering the columns $\tilde A_i'$ to that of recovering $\tilde \mu_i$.
  
  \paragraph{Step 1:  GMM Dependency Replacements.}
  In the following two claims, we consider alternative lower bounds for $N$ for recovering column estimators $\tilde A_1', \dotsc, \tilde A_m'$ which are $\epsilon^*$-close up to sign to the columns of $A'$.
  In particular, so long as we use at least as many samples of $X'$ as in \eqref{eq:sample_bound_for_err_prop0} when calling \textbf{UnderdeterminedICA}, then $A'$ will be recovered with the desired precision with probability $1 - \delta$.
\begin{claim*}
The $\poly(\norm{\mu_{\max}'}^{d^2}, d^{d^2})$ dependence in \eqref{eq:sample_bound_for_err_prop0} can be replaced by a $\poly(\norm{\mu_{\max}}^{d^2}, \allowbreak d^{d^2})$ dependence.
 \end{claim*}
\begin{innerproof}[Proof of Claim.] 
    By construction, $\mu_{\max}' = \left( \begin{array}{c} \mu_{\max} \\ 1 \end{array} \right)$.
    By the triangle inequality,
    \[
      \norm{\mu_{\max}'}^{d^2} \leq (\norm{\mu_{\max}} + 1)^{d^2}
    \]
    where $(\norm{\mu_{\max}} + 1)^{d^2}$ is a polynomial $q$ of $\norm{\mu_{\max}}$ with coefficients bounded by $(d^2)^{d^2} = d^{2d^2} = \allowbreak \poly(d^{d^2})$.  
   The maximal power of $\norm{\mu_{\max}}$ in $q(\norm{\mu_{\max}})$ is $d^{d^2}$. It follows that $q(\norm{\mu_{\max}}) = \poly(\norm{\mu_{\max}}^{d^2}, d^{d^2})$. 
\end{innerproof} \processifversion{vCOLT}{\vskip -28pt\phantom{hi}}
\begin{claim*}
The $\poly(1/\sigma_m(A'^{\odot d/2})^d)$ in \eqref{eq:sample_bound_for_err_prop0} can be replaced by a $\poly((\frac{\norm{\mu_{\max}} + 1}{\norm{\mu_{\min}}})^{d^2}, \ \allowbreak 1/\sigma_m(B^{\odot d/2})^{d})$ dependence.
\end{claim*}
\begin{innerproof}[Proof of Claim] 
    First define $\underbar A'$ to be the unnormalized version of $A'$.  
    That is, $\underbar A'_i := \mu_i'$.
    Then, $\underbar A' = A' \diag{\norm{\mu_1'}, \dotsc, \norm{\mu_m'}}$ implies $\underbar A'^{\odot d/2} = A'^{\odot d/2} \diag{ \norm{\mu_1'}^{d/2}, \dotsc \norm{\mu_m'}^{d/2}}$.
    Thus, $\sigma_m(\underbar A'^{\odot d/2}) \leq \sigma_m(A'^{\odot d/2})\norm{\mu_{\max}'}^{d/2}$.
    
    Next, we note that $\underbar A' = \left( \begin{array}{c} A \\ \mathbf 1 \end{array}\right)$ where $\mathbf 1$ is an all ones row vector.
    It follows that the rows of $A^{\odot d/2}$ are a strict subset of the rows of $\underbar A'^{\odot d/2}$.  Thus,
    \[
      \sigma_m(A^{\odot d/2}) 
        = \inf_{\norm{u} = 1} \norm{A^{\odot d/2}u}
        \leq \inf_{\norm{u} = 1} \norm{\underbar A'^{\odot d/2}u}
        = \sigma_m(\underbar A'^{\odot d/2}) \ .
    \]
    
    Finally, we note that $B = A \diag{\frac 1 {\norm{\mu_1}}, \dotsc, \frac 1 {\norm{\mu_m}}}$ and $B^{\odot d/2} = A^{\odot d/2} \diag{\frac 1 {\norm{\mu_1}^{d/2}}, \dotsc, \frac 1 {\norm{\mu_m}^{d/2}}}$.
    It follows that $\sigma_m(B^{\odot d/2}) \leq \sigma_m(A^{\odot d/2})\frac 1 {\norm{\mu_{\min}}^{d/2}}$.
    Chaining together inequalities yields:  
    \begin{align*}
      \sigma_m(B^{\odot d/2}) &\leq \frac{\norm{\mu_{\max}'}^{d/2}}{\norm{\mu_{\min}}^{d/2}} \sigma_m(A'^{\odot d/2}) & \text{or alternatively} & & \frac{\norm{\mu_{\max}'}^{d/2}}{\norm{\mu_{\min}}^{d/2}} \cdot \frac 1 {\sigma_m(B^{\odot d/2})} &\geq \frac 1 {\sigma_m(A'^{\odot d/2})} \ .
    \end{align*}
    As $\mu_{\max}' = (\mu_{\max}^T\ 1)^T$, the triangle inequality implies $\norm{\mu'_{\max}} \leq \norm{\mu_{\max}} + 1$.
    As we require the dependency of at least $N > \poly((1/\sigma_m(A'^{\odot d/2}))^d)$ samples, it suffices to have the replacement dependency of $N > \poly((\frac{\norm{\mu_{\max}} + 1}{\norm{\mu_{\min}}})^{\frac d 2 \cdot d}(1/\sigma_m(B^{\odot d/2})^d) = \allowbreak \poly((\frac{\norm{\mu_{\max}} + 1}{\norm{\mu_{\min}}})^{d^2}(1/\sigma_m(B^{\odot d/2})^d)$ samples. 
\end{innerproof}
  
  Thus, it is sufficient to call \textbf{UnderdeterminedICA} with
  \begin{align}
    N &\geq \poly\left(n^{d}, m^{d^2}, 
      (\tau\sigma)^{d^2}, \norm{\mu_{\max}}^{d^2}, \Big(\frac {w_{\max}} {w_{\min}}\Big)^{d^2},
      d^{d^2}, 
      \left(\frac {{\norm{\mu_{\max}} + 1}}{\norm{\mu_{\min}}}\right)^{d^2}, \frac 1{\sigma_m(B^{\odot d/2})^{d}}, \frac 1 {\epsilon^*}, \frac 1 \delta \right) \label{eq:sample_bound_for_err_prop}
  \end{align}
  samples to achieve the desired $\epsilon^*$ accuracy on the returned estimates of the columns of $A'$ with probability $1-\delta$.

  \paragraph{Step 2:  Error propagation.}
  
    What remains to be
  shown is that an appropriate choice of $\epsilon^*$ enforces 
  $\norm{\mu_i - \tilde \mu_i} < \epsilon$ by propagating the error.
  
  Recall that $A_i' = \left(\begin{array}{c}\mu_i \\ 1\end{array}\right) \cdot \norm{\left(\begin{array}{c}\mu_i \\ 1\end{array}\right)}^{-1}$, making
  $A_i'(n+1) = \frac 1 {\sqrt{1 + \norm{\mu_i}^2}}$.  Thus,
  \begin{align} 
    A_i'(n+1) & \geq \frac 1 {\sqrt{1 + \norm{\mu_{\max}}^2}}
         \ . \label{eq:ineq1}
  \end{align}
  
  We have that:
\begin{align*}
  \norm{\mu_i - \tilde \mu_i} 
  &= \norm{\frac{A_i'(1:n)}{A_i'(n+1)} - \frac{\tilde A_i'(1:n)}{\tilde A_i'(n+1)}} \\
  &= \norm{\frac{A_i'(1:n)}{A_i'(n+1)} - \frac{\alpha_i\tilde A_i'(1:n)}{A_i'(n+1)} + \frac{\alpha_i\tilde A_i'(1:n)}{A_i'(n+1)} - \frac{\alpha_i\tilde A_i'(1:n)}{\alpha_i\tilde A_i'(n+1)}} \\
  & \leq \frac{\norm{A_i'(1:n)-\alpha_i\tilde A_i'(1:n)}}{\abs{A_i'(n+1)}} +
\frac{\norm{\tilde A_i'(1:n)}\abs{\alpha_i\tilde A_i'(n+1)-A_i'(n+1)}}{\abs{A_i'(n+1)\alpha_i\tilde A_i'(n+1)}} \\
  & \leq \epsilon^* \sqrt{1 + \norm{\mu_{\max}}^2} + \frac {\abs{\alpha_i\tilde A_i'(n+1)-A_i'(n+1)}}{\abs{A_i'(n+1)}\left[\abs{A_i'(n+1)} - \abs{\alpha_i\tilde A_i'(n+1) - A_i'(n+1)}\right]}
\end{align*}
which follows in part by applying \eqref{eq:ineq1} for the left summand and noting that $\tilde A_i'$ is a unit vector for the right summand, giving the bound $\norm{\tilde A_i'(1:n)} \leq 1$.  Continuing with the restriction that $\epsilon^* < \frac 1 2 \frac 1 {\sqrt{1 + \norm{\mu_{\max}}^2}}$,
\begin{align*}
  \norm{\mu_i - \tilde\mu_i} 
  & \leq \epsilon^* \sqrt{1 + \norm{\mu_{\max}}^2} + \frac {\epsilon^*\sqrt{1 + \norm{\mu_{\max}}^2}}{\left[\frac 1 {\sqrt{1 + \norm{\mu_{\max}}^2}} - \epsilon^*\right]} \\
  & \leq \epsilon^* \left(\sqrt{1 + \norm{\mu_{\max}}^2} + 2 (1 + \norm{\mu_{\max}}^2)\right).
\end{align*}

Then, in order to guarantee that $\norm{\mu_i - \tilde\mu_i} < \epsilon$, it suffices to choose $\epsilon^*$ such that 
\[
  \epsilon^*\left(\sqrt{1 + \norm{\mu_{\max}}^2} + {2 (1 + \norm{\mu_{\max}}^2)}\right) \leq \epsilon, 
\]
which occurs when
\begin{equation} \label{eq:eps*choice}
  \epsilon^* \leq \frac {\epsilon}{\left(\sqrt{1 + \norm{\mu_{\max}}^2} + {2 (1 + \norm{\mu_{\max}}^2)}\right)}  \ .
\end{equation}
As $\epsilon < \frac 1 2$, the restriction $\epsilon^* < \frac 1 2 \sqrt{1 + \norm{\mu_{\max}^2}}$ holds automatically for the choice of $\epsilon^*$ in \eqref{eq:eps*choice}. 
The sample bound from \eqref{eq:sample_bound_for_err_prop} contains the dependency $N > \poly(\frac 1 {\epsilon^*}, \norm{\mu_{\max}}^{d^2})$.  Propagating the error gives a replacement dependency of $N > \poly\left(\frac 1 \epsilon, \sqrt{1 + \norm{\mu_{\max}}^2}, \norm{\mu_{\max}}^{d^2}\right) = \poly( \frac 1 \epsilon, \norm{\mu_{\max}}^{d^2})$ as $d$ is non-negative.  This propagated dependency is reflected in \eqref{eq:poly_bound_ideal_means}.
\end{proof}


\subsection{Distance of the Sampled Model to the Ideal Model}
\label{subsec:TotalVarDist}

An important part of the reduction is that the coordinates of $S$ are mutually independent. Without the threshold $\tau$, this is true (c.f. Lemma \ref{lem:Poisson-independence}).
However, without the threshold, one cannot know how to add more noise so that the total noise on each sample is iid.
We show that we can choose the threshold $\tau$ large enough that the samples still come from a distribution with arbitrarily small total variation distance to the one with truly independent coordinates.

\begin{lemma} \label{lem:poisson-threshold}
Fix $\delta > 0$. Let $S \sim \Poisson(\lambda)$ for $\lambda \geq \ln \delta$.  Let $b = e \lambda$, If $\tau > e \lambda$, $\tau \geq 1$, and $\tau \geq \ln(1/\delta) - \lambda$, then $\prob{S > \tau} < \delta$.
\end{lemma}

\begin{proof}
By the Chernoff bound (See Theorem A.1.15 in \cite{alon2004probabilistic}),
\[
  \prob{S > \lambda(1+\epsilon)} \leq \left( e^\epsilon (1+\epsilon)^{-(1+\epsilon)} \right)^\lambda.
\]
For any $\tau > \lambda$, letting $\epsilon = \tau/\lambda - 1$, we get $$\prob{S > \tau} \leq \frac{e^{-\lambda} (e\lambda)^{\tau}}{\tau^\tau}.$$
To get $\prob{S > \tau} < \delta$, it suffices that $\tau - \tau \log_b \tau \leq \log_b(\delta e^\lambda)$. Note that $$\tau(1 - \log_b \tau) = \tau - \tau \log_b \tau =
\log_b \big(b^{\tau} (1/\tau)^{\tau}\big).$$ If $\tau - \tau \log_b \tau \leq \log_b \left( \delta e^{\lambda} \right)$, then we have
\begin{align*}
\log_b \big(b^{\tau} (1/\tau)^{\tau}\big) &\leq \log_b(\delta e^{\lambda})
\end{align*}
which then implies it suffices that
\begin{align*}
\frac{b^{\tau}}{\tau^{\tau}} = \frac{(e\lambda)^\tau}{\tau^\tau} \leq \lambda^\tau / \tau^\tau \leq (1/e)^\tau &\leq \delta e^{\lambda}
\end{align*}
which holds for $\tau \geq \ln \left( \frac{1}{\delta e^\lambda} \right) = \ln(1/\delta) - \lambda$, giving the desired result.
\end{proof}

\begin{lemma} \label{lem:union-bound}
Let $N, \delta > 0$, $N \in \N$, and $T_1, T_2, \dots, T_N$ be iid with distribution $\Poisson(\lambda)$.
If $\tau \geq \ln (N/\delta) - \lambda$ then $$\prob{ \bigcup_{i} \left\{ T_i > \tau \right\}} < \delta.$$
\end{lemma}

\begin{proof}
By Lemma \ref{lem:poisson-threshold} $\tau \geq \ln (N/\delta) - \lambda$ implies $\prob{T_i > \tau} < \delta/N$ for every $i$.
The union bound gives us the desired result.
\end{proof}

It should now be easy to see that if we choose our threshold $\tau$ large enough, our samples can be statistically close (See Appendix \ref{app:total-variation}) to ones that would come from the truly independent distribution.
This claim is made formal as follows:

\begin{lemma} \label{lem:total-variation}
Fix $\delta > 0$.
Let $\tau > 0$.
Let $F$ be a Poisson distribution with parameter $\lambda$ and have corresponding density $f$.
Let $G$ be a discrete distribution with density $g(x) = f(x) / F(\tau)$ when $0 \leq x \leq \tau$ and 0 otherwise.
Then $\d_{TV}(F,G) = 1 - F(\tau)$.
\end{lemma}

\begin{proof}
Since we are working with discrete distributions, we can write $$d_{TV}(F,G) = \frac{1}{2} \sum_{i=0}^{\infty}| f(i) - g(i) |.$$
Then we can compute
\begin{align*}
\d_{TV}(F,G) &= \frac{|F(\tau) - 1|}{2F(\tau)} \sum_{i=0}^{\tau} f(i) + \frac{1}{2} \sum_{i=\tau+1}^{\infty} f(i) = \frac{|F(\tau) - 1|}{2} + \frac{1 - F(\tau)}{2} = 1 - F(\tau) \ .
\end{align*}
\end{proof}

\subsection{Proof of Theorem \ref{thm:correctness}}
\label{subsec:CorrectnessProof}

We now show that after the reduction is applied, we can use the \textbf{UnderdeterminedICA} routine given in \cite{GVX} to learn the GMM.
Instead of requiring exact values of each parameter, we simply require a bound on each.
The algorithm remains polynomial on those bounds, and hence polynomial on the true values.

\begin{proof}
The algorithm is provided parameters:
Covariance matrix $\Sigma$,
upper bound on tensor order $d$,
access to samples from a mixture of $\means$ identical spherical Gaussians in $\R^{\dim}$ with covariance $\Sigma$,
confidence $\delta$,
accuracy $\epsilon$,
upper bound $w \geq \max_{i}(w_i) / \min_{i}(w_i)$,
upper bound on the norm of the mixture means $u$,
lower bound $v$ so $0 < b \leq \sigma_m(A^{\odot d/2})$, and
$r \geq \big(\max_i\norm{\mu_i} + 1)/(\min_i\norm{\mu_i})\big)$.

The algorithm then needs to fix the number of samples $N$, sampling threshold $\tau$, Poisson parameter $\lambda$, and two new errors $\delta_1$ and $\delta_2$ so that $\delta_1 + \delta_2 \leq \delta$.
For simplicity, we will take $\delta_1 = \delta_2 = \delta/2$. Then fix $\sigma = \sup_{v \in S^{n-1}}\sqrt{\Var(v^T\eta(1))}$ for $\eta(1) \sim \Norm(0, \Sigma)$.
Recall that $B$ is the matrix whose $i$th column is $\mu_i/\norm{\mu_i}$. Let $A'$ be the matrix whose $i$th column is $(\mu_i, 1) / \norm{(\mu_i, 1)}$.

\paragraph{Step 1}
Assume that after drawing samples from Subroutine \ref{sub:independentsamples}, the signals $S_i$ are mutually independent (as in the ``ideal'' model given by (\ref{eq:full-ideal-model})) and the mean matrix $B$ satisfies $\sigma_m(B^{\odot d/2}) \geq b > 0$.
Then by Theorem \ref{thm:MeanRecoveryIdeal}, with probability of error $\delta_1$, the call to \textbf{UnderdeterminedICA} in Algorithm \ref{alg:reduction} recovers the columns of $B$ to within $\epsilon$ and up to a permutation using $N$ samples of complexity
\begin{align*} p\left(\tau^{d^2}, \Theta \right) = \poly\left(n^{d}, m^{d^2},
      (\tau \sigma)^{d^2}, u^{d^2}, w^{d^2},
      d^{d^2}, r^{d^2}, 1/b^{d}, 1/\epsilon, 1/\delta_1 \right)
\end{align*}
where $p(\tau^{d^2}, \Theta)$ is the bound on $N$ promised by Theorem \ref{thm:MeanRecoveryIdeal} and $\Theta$ is all its arguments except the dependence in $\tau$.
So then we have that with at least $N$ samples in this ``ideal'' case, we can recover approximations to the true means in $\R^\dim$ up to a permutation and within $\epsilon$ distance.

\paragraph{Step 2}
We need to show that after getting $N$ samples from the reduction, the resulting distribution is still close in total variation to the independent one.
We will choose a new $\delta' = \delta_2/(2N)$.
Let $R \sim \Pois(\lambda)$.
Given $\delta'$, Lemma \ref{lem:total-variation} shows that for $\tau \geq \ln(1/\delta') - \lambda$, with probability $1-\delta'$, $R \leq \tau$.

Take $N$ iid random variables $X_1, X_2, \dots, X_N$ from the $\Poisson(\lambda)$ distribution.
Let $G$ be a distribution given by density function $g(x) = (f(x) \mathds{1}_{0 \leq x \leq \tau})/F(\tau)$.
Let $Y_1, Y_2, \dots, Y_N$ be iid random variables with distribution $G$.
Denote the joint distribution of the $X_i$'s by $F'$ with density $f'$, and the joint distribution of the $Y_i$'s as $G'$ with density $g'$.
By the union bound and the fact that total variation distance satisfies the triangle inequality, $$\d_{TV}(F', G') \leq \sum_{i=1}^{N} \d_{TV}(F,G) = N d_{TV}(F,G).$$
Then for our choice of $\tau$, by Lemma \ref{lem:poisson-threshold} and Lemma \ref{lem:total-variation}, we have $$\d_{TV}(F',G') \leq N d_{TV}(F, G) = N \prob{X_1 > \tau} \leq N\delta' = \delta_2/2. $$

By the same union bound argument, the probability that the algorithm fails (when $R > \tau$) is at most $\delta_2/2$, since it has to draw $N$ samples.
So with high probability, the algorithm does not fail; otherwise, it still does not take more than polynomial time, and will terminate instead of returning a false result.

\paragraph{Step 3}
We know that $N$ is at least a polynomial which can be written in terms of the dependence on $\tau$ as $p(\tau^{d^2}, \Theta)$.
This means there will be a power of $\tau$ which dominates all of the $\tau$ factors in $p$, and in particular, will be $\tau^{Cd^2}$ for some $C$.
It then suffices to choose $C$ so that $p\left(\tau^{d^2}, \Theta \right) \leq \tau^{Cd^2} q(\Theta) \leq N$, where
\begin{align} q(\Theta) & = \poly\left(n^{d}, m^{d^2},
      \sigma^{d^2}, u^{d^2}, w^{d^2},
      d^{d^2}, r^{d^2}, 1/b^{d}, 1/\epsilon, 1/\delta_1\right) \label{eq:q-theta}.
\end{align}
Then, with the proper choice of $\tau$ (to be specified shortly), from step 2 we have
$$p\left(\tau^{d^2}, \Theta \right) \leq \tau^{Cd^2}q(\Theta) \leq N = \frac{\delta_2}{\delta'} \leq \frac{\delta_2 \tau^\tau e^\lambda}{(e\lambda)^\tau} = \frac{\delta \tau^\tau e^\lambda}{2(e\lambda)^\tau}.$$
Since $\lambda \geq 1$ it suffices to choose $\tau$ so that
\begin{equation} \label{eq:tau-sample-bound}
\frac{2}{\delta}q(\Theta)\tau^{Cd^2} \leq \frac{\tau^\tau}{\tau^{Cd^2}(e\lambda)^\tau}.
\end{equation}
Finally, we claim that
$$
\tau = 4\big(\log(2/\delta) + \log(q(\Theta))\big)\max\left((e\lambda)^2, 4Cd^2\right) = O\left((\lambda^2 + d^2)\log \frac{q(\Theta)}{\delta}\right)
$$
is enough for the desired bound on the sample size. Observe that $4(\log(2/\delta) + \log(q(\Theta))) \geq 1$.

An useful fact is that for general $x, a, b \geq 1$, $x \geq \max(2a, b^2)$ satisfies $x^a \leq x^x/b^x$.
This captures the essence of our situation nicely.
Letting $e\lambda$ play the role of $b$, $Cd^2$ play the role of $a$ and $x$ play the role of $\tau$, to satisfy (\ref{eq:tau-sample-bound}), it suffices that
\begin{align*}
\frac{2}{\delta}q(\Theta) &\leq \frac{\tau^{\tau/2} \tau^{\tau/4} \tau^{\tau/4}}{\tau^{Cd^2}(e\lambda)^2}.
\end{align*}
We can see that $\tau^{\tau/2} \geq (e\lambda)^2$ and $\tau^{\tau/4} \geq \tau^{Cd^2}$ by construction.
But we also get $\tau/4 \geq \log(2/\delta) + \log{q(\Theta)}$ which implies $\tau^{\tau/4} \geq e^{\tau/4} \geq \frac{2}{\delta} q(\Theta)$.
Thus for our choice of $\tau$, which also preserves the requirement in Step 2, there is a corresponding set of choices for $N$, where the required sample size remains polynomial as
\begin{align*}
& \poly\left(n^{d}, m^{d^2},
      (\tau \sigma)^{d^2}, u^{d^2}, w^{d^2},
      d^{d^2}, r^{d^2}, 1/b^{d}, 1/\epsilon, 1/\delta \right)
\end{align*}
where we used the bound $q(\Theta) \leq (n^d m^{d^2} \sigma^{d^2} u^{d^2} w^{d^2} (d+1)^{d^2} r^{d^2} / b^{d} \delta_1 \epsilon)^{O(1)}$.
By the choice of $\tau$, one can absorb $\tau^{d^2}$ into the above $\poly(\cdot)$ expression, giving the result.
\end{proof}

\section{Lemmas on the Poisson Distribution} \label{sec:Poisson-lemmas}
The following lemmas are well-known; see, e.g., \cite{bookDasgupta}. We provide proofs for 
completeness.
\begin{lemma} \label{lem:appendix-Poisson-basic}
If $X \sim \Pois(\lambda)$ and $Y |_{X=x} \sim \Bin(x,p)$ then $Y \sim \Pois(p \lambda)$.
\end{lemma}

\begin{proof}
\begin{align*}
\prob{Y=y}
&= \sum_{x: x \geq y}^{\infty} \prob{Y = y \suchthat X = x} \prob{X = x} \\
&= \sum_{x: x \geq y}^{\infty} {x \choose y} p^y (1-p)^{x-y} \frac{\lambda^x e^{-\lambda}}{x!} \\
&= p^y e^{-\lambda} \sum_{x: x \geq y}^{\infty} \frac{\lambda^x}{x!} {x \choose y} (1-p)^{x-y} \\
&= \frac{(p \lambda)^y e^{-\lambda}}{y!} \sum_{x: x \geq y}^{\infty} \frac{(\lambda(1-p))^{x-y}}{(x-y)!} \\
&= \frac{(p \lambda)^y e^{-\lambda}}{y!} e^{(1-p)\lambda} \\
&= \frac{(p \lambda)^y e^{-p \lambda}}{y!}.
\end{align*}
\end{proof}

\begin{lemma} \label{lem:appendix-poisson-independence}
Fix a positive integer $k$, and let $p_i \geq 0$ be such that $p_1 + \dotsb + p_k =1$. If $X \sim \Pois(\lambda)$ and
$(Y_1, \ldots, Y_k)|_{X = x} \sim \Multinom(x; p_1, \ldots, p_k)$ then $Y_i \sim \Pois(p_i \lambda)$ for all $i$ and
$Y_1, \ldots, Y_k$ are mutually independent.
\end{lemma}

\begin{proof}
  The first part of the lemma (i.e., $Y_i \sim \Pois(p_i \lambda)$ for all $i$) follows from Lemma~\ref{lem:appendix-Poisson-basic}.
For the second part, let's prove it for the binomial case ($k=2$); the general case is similar.
\begin{align*}
\prob{Y_1 = y_1, Y_2 = y_2} &= \prob{Y_1 = y_1, Y_2 = y_2 \suchthat X = y_1+y_2} \prob{X = y_1+y_2} \\
&= {y_1+y_2 \choose y_1} p^{y_1}(1-p)^{y_2} \cdot \frac{\lambda^{y_1+y_2} e^{-\lambda}}{(y_1+y_2)!} \\
&= \frac{(p\lambda)^{y_1} e^{-p\lambda}}{y_1!} \cdot \frac{((1-p)\lambda)^{y_2} e^{-(1-p)\lambda}}{y_2!} \\
&= \prob{Y_1 = y_1} \cdot \prob{Y_2 = y_2}.
\end{align*}
\end{proof}

\section{Properties of Cumulants}\label{sec:cumulant-properties}
The following properties of multivariate cumulants are well known and are largely inherited from the definition of the cumulant generating function:
\begin{itemize}
  \item (Symmetry) Let $\sigma$ give a permutation of $k$ indices.  
  Then, $\cumtns Y {i_1, \cdots, i_\ell} = \cumtns Y {\sigma(i_1), \cdots, \sigma(i_\ell)}$.
  \item (Multilinearity of coordinate random variables)
  Given constants $\alpha_1, \cdots, \alpha_\ell$, then
  \[
    \cum(\alpha_1 Y_{i_1}, \cdots, \alpha_\ell Y_{i_\ell}) = 
      \left(\prod_{i=1}^\ell \alpha_i\right)\cum(Y_{i_1}, \cdots, Y_{i_\ell}) \ .
  \]
  Also, given a scalar random variable $Z$, then
  \[
    \cum(Y_{i_1}+Z, Y_{i_2}, \cdots, Y_{i_\ell}) = 
      \cum(Y_{i_1}, Y_{i_2}, \cdots, Y_{i_\ell}) + \cum(Z, Y_{i_2}, \cdots, Y_{i_\ell})
  \]
  with symmetry implying the additive multilinear property for all other coordinates.
  \item (Independence)
  If there exists $i_j, i_k$ such that $Y_{i_j}$ and $Y_{i_k}$ are independent
  random variables, then the cross-cumulant $\cumtns Y {i_1, \cdots, i_\ell} = 0$.
  Combined with multilinearity, it follows that when there are two independent random vectors $Y$ and $Z$, then $\cumtns {Y+Z} {} = \cumtns Y {} + \cumtns Z {}$.
  \item (Vanishing Gaussians)
  When $\ell \geq 3$, then for the Gaussian random variable $\eta$, $\cumtns \eta {} = 0$.
\end{itemize}

\section{Bounds on Stirling Numbers of the Second Kind}
The following bound comes from \cite[Theorem 3]{Stirling}.
\begin{lemma} \label{lem:StirlingBound_nr}
  If $n \geq 2$ and $1 \leq r \leq n-1$ are integers, then $\stirling{n}{r} \leq \frac 1 2 {n \choose r} r^{n-r}$.
\end{lemma}

From this, we can derive a somewhat looser bound on the Stirling numbers of the second kind which does not depend on $r$:
\begin{lemma} \label{lem:StirlingBound_n}
  If $n, r \in \Z^+$ such that $r \leq n$, then $\stirling{n}{r} \leq n^{n-1}$.
\end{lemma}
\begin{proof}
  The Stirling number $\stirling{n}{k}$ of the second kind gives a count of the number of ways of splitting a set of $n$ labeled objects into $k$ unlabeled subsets.
  In the case where $r = n$, then $\stirling{n}{r} = 1$  
  As $n \geq 1$, it is clear that for these choices of $n$ and $r$, $\stirling{n}{r} \leq n^{n-1}$.
  By the restriction $1 \leq r \leq n$, when $n=1$, then $n=r$ giving that $\stirling{n}{r} = 1$.  As such, the only remaining cases to consider are when $n \geq 2$ and $1 \leq r \leq n-1$, the cases where Lemma \ref{lem:StirlingBound_nr} applies.
  
  When $n \geq 2$ and $1 \leq r \leq n-1$, then
  \begin{align*}
\stirling{n}{r} &\leq \frac 1 2 {n \choose r}r^{n-r} = \frac 1 2 \frac {n!}{r!(n-r)!} r^{n-r}
      \leq \frac 1 2 n^{r}r^{n-r-1} < \frac 1 2 n^r n^{n-r-1}
       = \frac 1 2 n^{n-1} \ ,
  \end{align*}
  which is slightly stronger than the desired upper bound.
\end{proof}

\section{Values of Higher Order Statistics}\label{sec:GaussMoments} \label{app:HOS}

In this appendix, we gather together some of the explicit values for higher order
statistics of the Poisson and Normal distributions required for the analysis
of our reduction from learning a Gaussian Mixture Model to learning an ICA model
from samples.

\begin{lemma}[Cumulants of the Poisson distribution] \label{lem:PoissonCumulants}
  Let $X \sim \Poisson(\lambda) $.  
  Then, $\cum_{\ell}(X) = \lambda$ for every positive integer $\ell$.
\end{lemma}
\begin{proof}
  The moment generating function of the Poisson distribution is given by $M(t) = \exp(\lambda (e^{t} - 1))$.
  The cumulant generating function is thus $g(t) = \log(M(t)) = \lambda(e^{t} - 1)$.
  The $\ell$\textsuperscript{th} derivative $(\ell \geq 1)$ is given by $g^{(\ell)}(t) = \lambda e^{t}$.
  
  By definition, $\cum_{\ell}(X) = g^{(\ell)}(0) = \lambda$.
\end{proof}

\begin{lemma}[Absolute moments of the Gaussian distribution] \label{lem:GaussianMoments}
The absolute moments of the Gaussian random variable $\eta \sim N(0, \sigma^2)$\anonnote{xxx} are given by:
\[
\E{\abs{\eta}^\ell} =  
\begin{cases}
    	\sigma^\ell \frac{\ell!}{2^{\sfrac \ell 2}(\sfrac \ell 2)!} & \text{if $\ell$ is even} \\
    	\sigma^\ell 2^{\sfrac \ell 2}(\frac{\ell-1}2)! \frac 1 {\sqrt \pi } & \text{if $\ell$ is odd}.
\end{cases}
\]
\end{lemma}
The case that $\ell$ is even in Lemma \ref{lem:GaussianMoments} is well known, and can be found for instance in \cite[Section 3.4]{Kendall94}. For general $\ell$, it is known (see \cite{Winkelbauer2012}) that 
  \[
    \E{\abs{\eta}^{\ell}} = \sigma^\ell 2^{\sfrac \ell 2} \Gamma\left(\frac {\ell+1} 2 \right) \frac 1 {\sqrt \pi }  \ .
  \]
  When $\ell$ is odd, $\frac{\ell + 1} 2$ is an integer, allowing the Gamma function to simplify to a factorial:  $\Gamma\left(\frac {\ell + 1} 2\right) = \left(\frac{\ell - 1}{2}\right)!$.
  This gives the case where $\ell$ is odd in Lemma \ref{lem:GaussianMoments}.

\section{Total Variation Distance} \label{app:total-variation}

Total variation is a type of statistical distance metric between probability distributions.
In words, the total variation between two measures is the largest difference between the measures on a single event.
Clearly, this distance is bounded above by 1.

For probability measures $F$ and $G$ on a sample space $\Omega$ with sigma-algebra $\Sigma$, the total variation is denoted and defined as:
$$\d_{TV}(F,G) := \sup_{A \in \Sigma} |F(A) - G(A)|. $$

Equivalently, when $F$ and $G$ are distribution functions having densities $f$ and $g$, respectively, $$ \d_{TV}(F,G) = \frac{1}{2} \int_{\Omega} |f - g| d\mu $$
where $\mu$ is an arbitrary positive measure for which $F$ and $G$ are absolutely continuous.

More specifically, when $F$ and $G$ are discrete distributions with known densities, we can write $$ \d_{TV}(F,G) = \frac{1}{2} \sum_{k=0}^{\infty} |f(k) - g(k)| $$
where we choose $\mu$ that simply assigns unit measure to each atom of $\Omega$ (in this case, absolute continuity is trivial since $\mu(A) = 0$ only when $A$ is empty and thus $F(A)$ must also be 0). For more discussion, one can see Definition 15.3 in \cite{nielsen1997introduction} and Sect. 11.6 in \cite{royden1988real}.



\section{Sketch for the proof of Theorem~\ref{thm:low-dim-identifiability-ica}}
\label{sec:sketch-pf-thm-ICA-bound}

\paragraph{Lower bound for ICA.} 
We can use our Poissonization technique to embed difficult instances of learning GMMs into the ICA setting to prove that 
ICA is information-theoretically hard when the observed dimension $n$ is a constant using the lower bound for learning GMMs. We are not 
aware of any existing lower bounds in the literature for this problem. We only
provide an informal outline of the argument. 

Theorem~\ref{thm:low-dim-identifiability} gives us two GMMs $p$ and $q$ of identity covariance Gaussians that are exponentially close with respect to $k^2$ (the number of points used to generate the Gaussian means) in $L^1$ distance but far in parameter distance.
We apply the basic reduction from Section~\ref{sec:reduction} with $\lambda$ set to the number of Gaussian means associated with the respective GMMs $p$ and $q$ to obtain the ideal noisy ICA models $X_p = A_pS_p + \eta(\tau)$ and $X_q = A_qS_q + \eta(\tau)$ (model (1) from Section~\ref{sec:reduction}).
Then, we let $S_p$ and $S_q$ take on the scaling information of the ICA model by replacing $S_{pi}$ and $S_{qj}$ by $\norm{A_{pi}}S_{pi}$ and $\norm{A_{qj}}S_{qj}$ respectively, and replacing the columns of $A_p$ and $A_q$ with their unit-normalized versions.
While Theorem~\ref{thm:low-dim-identifiability} is proven in the setting where Gaussian means are drawn uniformly at random from the unit hypercube, it can be reformulated to have Gaussian means drawn uniformly at random from the unit ball.
Under such a reformulation, the columns of $A_q$ and $B_q$ are chosen from a set of $k^2$ points taken uniformly from the unit sphere $S^{n-1} \subset \R^n$, which are thus well separated with high probability.

Recall that $R_p = \sum_{i} S_{pi}$ and $R_q = \sum_{i} S_{qi}$ are Poisson distributed with parameters $m_p$ and $m_q$ denoting the number of columns of $A_p$ and $A_q$ respectively.
Lemma~\ref{lem:poisson-threshold} implies that for a choice of $\tau_p$ which is linear in $m_p$, the probability of a draw with $R_p > \tau_p$ is exponentially small, and similarly for $\tau_q$.  
In particular, we choose $\tau = \max(\tau_p, \tau_q)$ for the above ICA models.

Now since the $L^1$ (and hence total variation) distance between $p$ and $q$ is exponentially small in $k^2$ (upper bound on the number of components), 
the distance between the two resulting ICA models produced by the reduction is also exponentially small (specifically, the total variation distance between the random variables $X_p$ and $X_q$).
To see this, we must condition on several cases.  First, conditioning either model on $R > \tau$, we have that $\Pr(R > \tau)$ is exponentially small, and hence its contribution to the overall total variation distance between $X_p$ and $X_q$ is exponentially small.
Conditioning on $R = z$ where $z \in \{0, 1, \ldots, \tau\}$, then the facts that $p$ and $q$ are close in total variation distance and that total variation distance satisfies a version of the triangle inequality (that is for random variables $C, D, E, F$, $d_{TV}(C + D, E + F) \leq d_{TV}(C, E) + d_{TV}(D, F)$) imply that by viewing $X_p$ (and similarly for $X_q$) as the sum of $z$ draws from the distribution $p$ and $\tau-z$ draws from the additive Gaussian noise distribution, the total variation distance between $X_p$ and $X_q$ conditioned on $R=z$ is still exponentially small.
Thus, the non-conditional distributions of $X_p$ and $X_q$ will be exponentially close in $k$ in total variation distance.
In particular, the sample complexity of distinguishing between $X_p$ and $X_q$ is exponential in $k$.

One can also interpret ICA with Gaussian noise as ICA without noise by treating the noise as extra signals: If $X = AS + \eta$
is an ICA model where $A \in \R^{n \times m}$ and $\eta \in \R^n$ is spherical Guassian noise, then by defining 
$A' := [A | I_n]$, and $S' := [S^T, \eta^T]^T$ we get
$X = A'S'$ which is a noiseless model with some of the signals being Gaussian. In such cases, algorithms (such as that of \cite{GVX}) 
are able to still recover
the non-Gaussian portion $A$ of $A'$. Our result shows that such algorithms cannot be efficient if the observations are in small 
dimensions (i.e. $n$ is small). 

\section{}

\subsection{Underdetermined ICA theorem}\label{subsec:UICA_noisy}

\begin{theorem}[\cite{GVX}]\label{thm:UICA_noisy}\lnote{move to appendix}
Let a random vector $x \in \R^n$ be given by an underdetermined ICA model with unknown Gaussian noise $x= As + \eta$ where $A \in \R^{n \times m}$ has unit norm columns, and both $A$ and the covariance matrix $\Sigma \in \R^{n \times n}$
are unknown. Let $d \in 2 \N$ be such that $\sigma_m(A^{\odot d/2}) > 0$. Let $k > d$ be such that
 for each $s_i$, there is a $k_i$ satisfying $d < k_i \le k$ and $\abs{\cum_{k_i}(s_i)} \ge
  \Delta$, and $\E{ \abs{s_i}^k} \le M$.  
Moreover, suppose that the noise also satisfies the same moment condition: $\E{\abs{\angles{u, \eta_i}}^k} \le M$ for any unit
vector $u \in \R^n$ (this is satisfied if we have $k! \sigma^k \leq M$ where $\sigma^2$ is the maximum eigenvalue of 
$\Sigma$).  
Then algorithm \textbf{UnderdeterminedICA} returns a set of $n$-dimensional vectors $(\tilde A_i)_{i=1}^m$ so that for some permutation $\pi$ of $[m]$ and signs $\alpha_i \in \{-1, 1\}$ we have $\norm{\alpha_i \tilde A_{\pi(i)} - A_i} \leq \eps$ for all $i \in [m]$.
Its sample and time complexity are $\text{poly}\left( n^{k} , m^{k^2}, M^k, 1/ \Delta^k, 1/\sigma_m(A^{\odot d/2})^k, 1/\eps, 1/\delta \right)$.
\end{theorem}

\subsection{Rudelson-Vershynin subspace bound}\label{subsec:rudelson-vershynin}

\begin{lemma}[\processifversion{vstd}{Rudelson--Vershynin~}\cite{RudelsonVershynin}] \label{lem:RudelsonVershynin}
If $A \in \R^{n \times m}$ has columns $C_1, \ldots, C_m$, then denoting $C_{-i} = \spn{C_j: j \neq i}$, we have 

\begin{align*}
\frac{1}{\sqrt{m}} \min_{i \in [m]}\dist(C_i, C_{-i}) \leq \sigma_{\min}(A),
\end{align*}

where as usual $\sigma_{\min}(A) = \sigma_{\min(m,n)}(A)$.
\end{lemma}

\subsection{Carbery-Wright anticoncentration}\label{subsec:carbery-wright}

The version of the anticoncentration inequality we use is explicitly given in \cite{MOS} which
in turn follows immediately from \cite{CarberyWright}:

\begin{lemma}[\cite{MOS}]\label{lem:CarberyWright}
Let $Q(x_1, \ldots, x_n)$ be a multilinear polynomial of degree $d$. 
Suppose that $\Vr{Q} = 1$ when $x_i \sim \Normal(0,1)$ for all $i$. Then there exists an absolute constant $C$
such that for $t \in \R$ and $\epsilon > 0$,
\begin{align*}
\Pr_{(x_1, \ldots, x_n) \sim \Normal(0,I_n)}(\abs{Q(x_1, \ldots, x_n)-t} \leq \epsilon) \leq C d \epsilon^{1/d}. 
\end{align*}
\end{lemma}

\section{Recovery of Gaussian Weights} \label{sec:WeightRecovery}

\paragraph{Multivariate cumulant tensors and their properties.}
Our technique for the recovery of the Gaussian weights relies on the tensor properties of multivariate cumulants that have been used in the ICA literature.

Given a random vector $Y \in \R^n$, the moment generating function of $Y$ is defined as $M_Y(t) := \mathbb{E}_Y(\exp(t^TY))$.
The \textit{cumulant generating function} is the logarithm of the moment generating function: $g_{\ds Y}(t) := \log(\mathbb{E}_{\ds Y}(\exp(t^T Y))$.


Similarly to the univariate case, multivariate cumulants are defined using the Taylor expansion of the cumulant generating function.
We use both $\cumtns{Y}{j_1, \dots, j_\ell}$ and $\cum(Y_{j_1}, \dots, Y_{j_\ell})$ to denote the order-$\ell$ cross cumulant between the random variables $Y_{j_1}, Y_{j_2}, \dots, Y_{j_\ell}$.
Then, the cross-cumulants $\cumtns{Y}{j_1, \dots, j_\ell}$ are defined as the coefficients of the Taylor expansion of $g_{\ds Y}(t)$ around 0, and can be obtained using the formula $\cumtns{Y}{j_1, \dots, j_\ell} = \frac{\partial}{\partial t_{j_1}} \cdots \frac{\partial}{\partial t_{j_\ell}} g_{\ds Y}(t) \big|_{t=0}$.
When unindexed, $\cumtns{Y}{}$ will denote the full order-$\ell$ tensor containing all cross-cumulants, with the order of the tensor being made clear by context.
In the special case where $j_1 = \cdots = j_\ell = j$, we obtain the order-$\ell$ univariate cumulant $\cum_{\ell}(Y_j) = \cumtns{Y}{j, \dots, j}$ ($j$ repeated $\ell$ times) previously defined.
We will use some well known properties of multivariate cumulants, found in Appendix~\ref{sec:cumulant-properties}.

The most theoretically justified ICA algorithms have relied on the tensor structure of multivariate cumulants, including
the early, popular practical algorithm JADE \cite{CardosoS93}.
In the fully determined ICA setting in which the number source signals does not exceed the ambient dimension, the papers \cite{AroraGMS12} and \cite{Belkin2012} demonstrate that ICA with additive Gaussian noise can be solved in polynomial time and using polynomial samples.  The tensor structure of the cumulants was (to the best of our knowledge) first exploited in \cite{FOOBI} and later in \cite{BIOME} to solve underdetermined ICA.
Finally, \cite{GVX} provides an algorithm with rigorous polynomial time and sampling bounds \vnote{same exponential caveats as our proposed technique in this paper} for underdetermined ICA in the presence of Gaussian noise.

\paragraph{Weight recovery (main idea).}
Under the basic ICA reduction (see section~\ref{sec:reduction}) using the Poisson distribution with parameter $\lambda$, we have that $X = AS + \eta$ is observed such that $A = [\mu_1 | \cdots | \mu_m]$ and $S_i \sim \Poisson (w_i \lambda)$.
As $A$ has already been recovered, what remains to be recovered are the weights $w_1, \cdots, w_m$.
These can be recovered using the tensor structure of higher order cumulants.
The critical relationship is captured by the following Lemma:
\begin{lemma} \label{lem:cum_COV}
  Suppose that $X = AS + \eta$ gives a noisy ICA model.
  When $\cumtns{X}{}$ is of order $\ell > 2$, then $ \vec{\cumtns{X}{}} = A^{\odot \ell} (\cum_\ell(S_1), \dotsc, \cum_\ell(S_m))^T$.
\end{lemma}
\begin{proof}
  It is easily seen that the Gaussian component has no effect on the cumulant:
  \begin{equation*}
    \cumtns{X}{} = \cumtns{AS + \eta}{} = \cumtns{AS}{}  + \cumtns{\eta}{}
      = \cumtns{AS}{}
  \end{equation*}
  Then, we expand $\cumtns X {}$:
  \begin{align*}
    \cumtns{X}{i_1, \cdots, i_\ell} &= \cumtns{AS}{i_1, \cdots, i_\ell} =\cum((AS)_{i_1}, \cdots, (AS)_{i_\ell}) \\
    &= \cum\left(\sum_{j_1 = 1}^m A_{i_1j_1}S_{j_1}, \cdots, \sum_{j_\ell = 1}^m A_{i_\ell j_\ell}S_{j_\ell}\right) \\
    &= \sum_{j_1, \cdots, j_\ell \in [m]} \left(\prod_{k=1}^\ell A_{i_kj_k}\right) \cum(S_{j_1}, \cdots, S_{j_\ell}) & \text{by multilinearity}
  \end{align*}
  But, by independence, $\cum(S_{j_1}, \cdots, S_{j_m}) = 0$ whenever $j_1 = j_2 = \cdots = j_\ell$ fails to hold.
  Thus, 
  \begin{align*}
    \cumtns{X}{i_1, \cdots, i_\ell} &= \sum_{j=1}^m \left(\prod_{k=1}^\ell A_{i_kj}\right) \cum_\ell(S_j) = \sum_{j=1}^m \big((A_j)^{\otimes \ell}\big)_{i_1, \cdots, i_\ell} \cum_\ell(S_j)
  \end{align*}
  Flattening yields: $\vec{\cumtns{X}{}} = A^{\odot \ell}(\cum_\ell(S_1), \cdots, \cum_\ell(S_m))^T$.
\end{proof}



In particular, we have
that $S_i \sim \Poisson(w_i \lambda)$ with $w_i$ the probability of sampling from the $i$\textsuperscript{th} Gaussian.  Given knowledge of $A$ and the cumulants
of the Poisson distribution, we can recover the Gaussian weights.
  \begin{theorem}
    Suppose that $X = AS + \eta(\tau)$ is the unrestricted noisy ICA model from the basic reduction (see section~\ref{sec:reduction}).
    Let $\ell > 2$ be such that $A^{\odot \ell}$ has linearly independent columns, 
and let $(A^{\odot \ell})^\dagger$ be its Moore-Penrose pseudoinverse.  
    Let $\cumtns{X}{}$ be of order $\ell$.
    Then $\frac 1 \lambda (A^{\odot \ell})^\dagger \vec{\cumtns{X}{}}$ is the vector of mixing weights $(w_1, \ldots, w_m)^T$ of the Gaussian mixture model.
  \end{theorem}
  \begin{proof}
  From Lemma \ref{lem:PoissonCumulants}, $\cum_{\ell}(S_i) = \lambda w_i$.  
  Lemma \ref{lem:cum_COV} implies that  $\vec{\cumtns{X}{}} = \lambda A^{\odot \ell} (w_1, \dotsc, w_m)^T$.
  Multiplying on the left by $\frac 1 \lambda (A^{\odot \ell})^\dagger$ gives the result.
  \end{proof}


\end{document}